%% file: main.tex
\documentclass{article}

\usepackage[preprint]{neurips_2026}
\usepackage[utf8]{inputenc}
\usepackage[T1]{fontenc}
\usepackage{hyperref}
\hypersetup{
  hidelinks,
  pdftitle={SCARCE: Scalable Cascade Analysis for Rare-event Characterisation via Embeddings},
  pdfauthor={Yingjie Wang, Yi Dong, Edmund Lau, Jie Meng, Taylor T Johnson, Xiaowei Huang}
}
\usepackage{url}
\usepackage{booktabs}
\usepackage{amsfonts}
\usepackage{amsmath}
\usepackage{amssymb}
\usepackage{nicefrac}
\usepackage{microtype}
\usepackage{xcolor}
\usepackage{graphicx}
\usepackage[caption=false]{subfig}
\usepackage{algorithm}
\usepackage{algpseudocode}
\usepackage{enumitem}
\usepackage{amsthm}
\usepackage{float}  
\usepackage{comment}

\newtheorem{theorem}{Theorem}
\newtheorem{corollary}[theorem]{Corollary}

\newtheorem{assumption}{Assumption}

\newcommand{\mae}{\mathrm{MAE}}

\setlist{topsep=2pt,itemsep=1pt,parsep=1pt,partopsep=0pt}

\makeatletter
\def\thm@space@setup{\thm@preskip=4pt \thm@postskip=4pt}
\makeatother

\newcommand{\projname}{SCARCE}

\title{\projname: Scalable Cascade Analysis for Rare-event Characterisation via Embeddings}

\author{
  Yingjie Wang\textsuperscript{1},
  Yi Dong\textsuperscript{1},
  Edmund Lau\textsuperscript{2},
  Jie Meng\textsuperscript{3},
  Taylor T Johnson\textsuperscript{4},
  Xiaowei Huang\textsuperscript{1,*}\\[0.5em]
  \textsuperscript{1}University of Liverpool, UK\\
  \textsuperscript{2}UK AI Security Institute, UK\\
  \textsuperscript{3}Loughborough University, UK\\
  \textsuperscript{4}Department of Computer Science, Vanderbilt University, USA\\[0.3em]
  \textsuperscript{*}Corresponding author:
  \texttt{Xiaowei.Huang@liverpool.ac.uk}
}
\begin{document}

\maketitle

\begin{abstract}
Rare events govern the safety profile of modern AI systems, yet their probabilities are extremely difficult to estimate: direct Monte Carlo demands prohibitive sample budgets. Subset Simulation (SS) addresses this by decomposing a rare-event target probability into a product of moderate conditional probabilities defined by a sequence of nested intermediate events. However, classical SS requires a handcrafted scalar performance function whose sublevel sets specify those events, demanding detailed knowledge of the failure geometry and limiting transfer to new domains. We propose \projname{} (Scalable Cascade Analysis for Rare-event Characterisation via Embeddings), which replaces the performance function with learned latent representations and a family of geometric rulers that score proximity to failure regions. Adaptive thresholding on these scores constructs the nested intermediate events directly from data, eliminating domain-specific heuristics. We formalise \projname{} via a non-negative supermartingale construction, yielding a high-probability upper envelope on the estimate that remains valid even under early stopping, showing that data-driven event design preserves statistical validity. On MNIST misclassification, where dense Monte Carlo provides ground truth, \projname{} achieves ${\sim}400\text{--}500\times$ lower mean absolute error than grid-searched traditional SS while eliminating the latter's systematic over-counting. We then study LLM PAIR-style jailbreaks as a rare-event estimation problem with a fleet-level threat model and tunable adversarial fraction $\eta$. On Llama-Guard-3-8B hidden states, a PCA-based ruler attains $2.6\%$ mean relative error against finite-sample reference probabilities for $\eta \geq 10^{-3}$ -- within the references' own $27.9\%$ bootstrap relative half-width -- and transfers to a GCG-style corpus with $2.93\%$ relative error after refitting calibration. A directional KL criterion $\mathrm{KL}(p_{\text{good}} \,\|\, p_{\text{bad}})$ ranks rulers consistently with the resulting estimation error (Spearman $\rho = 0.83$).
\end{abstract}

\input{Sections/Section1_Introduction}

\input{Sections/Section2_RelatedWork}

\input{Sections/Section3_SSDD}
\input{Sections/Section4_Theory_V2}

\input{Sections/Section5_Vision}

\input{Sections/Section6_LLM_JB_V3}

\input{Sections/Section7_Conclusion}

\input{Sections/Acknowledgements}


\bibliography{biblio}
\bibliographystyle{plain}

\appendix
\input{Sections/Appendix_SS}
\input{Sections/Appendix_Rulers}
\input{Sections/Appendix_Martingale}
\input{Sections/Appendix_Runtime}
\input{Sections/Appendix_LLM_JB}
\input{Sections/Appendix_Ruler_Selection}
\input{Sections/Appendix_Plots}


\end{document}

%% file: Sections/Section1_Introduction.tex
\section{Introduction}
\label{sec:intro}

Rare events govern the safety profile of modern AI systems: misclassification in vision models, policy violations in large language models (LLMs), and silent safety-filter bypasses. They all share extremely small probability under typical operating conditions, yet carry catastrophic consequences.
Quantifying such probabilities is a prerequisite for safety certification, yet direct Simple Monte Carlo (SMC) is computationally intractable.
For a target probability $P_f \approx 10^{-5}$, an SMC estimator with 10\% relative error demands on the order of $10^{7}$ independent forward passes. Such a budget is already prohibitive for a single image classifier and entirely infeasible for multi-turn LLM inference.

Subset Simulation (SS)~\cite{au2001estimation,au2003subset} tackles this problem by decomposing an extreme probability into a product of moderate conditional probabilities evaluated over nested intermediate events.
This factorisation reduces the required sample count from $100/P_f$ to $\sim N\times|\log_{10} P_f|$~\cite{zuev2015subset}, where $N$ is the per-level sample size and $\log_{10}$ reflects the standard SS choice of a $0.1$ conditional probability per level. For example, to estimate $P_f \approx 10^{-5}$, SS requires roughly $10^{5}$ in total with $N = 20000$, while SMC needs $10^{7}$ samples. 
It has been applied to structural reliability for over two decades~\cite{song2009subset,zuev2015subset}.
However, classical SS relies on a handcrafted \emph{performance function} $g(x)$ whose level sets define the nested subsets.
Designing such a function presupposes detailed knowledge of the failure geometry. It requires domain expertise, does not transfer across 
problem areas, and becomes infeasible when the failure manifold is 
high-dimensional, disconnected, or non-convex in the latent space. This expert-driven design of intermediate events is the bottleneck of 
classical SS: the method scales with the target probability but not with 
domain complexity. 

The limitations of classical SS are sharply exposed in LLM jailbreak estimation.
At deployment scale, an LLM serves a \emph{user fleet}: a large population of queries dominated by benign use, with a small adversarial subpopulation. Jailbreaks are extremely rare in this regime, yet a single occurrence can erode user trust \cite{martell2024mitigative}.
Existing paradigms measure empirical attack success rates under a fixed set of adversarial templates~\cite{jailbreakbench2024,chao2023pair} or 
forecast deployment risk via volume scaling~\cite{jones2025forecasting}, neither of which estimates fleet-level violation probability under an 
explicit threat model.
What is missing is a principled rare-event framework that can answer: \emph{given a fleet in which a fraction $\eta$ of users are adversarial, what is the per-turn probability of a safety-policy violation?}
No natural scalar performance function exists here. The ``distance to jailbreak'' lives in the high-dimensional latent space of a safety-judge model, whose geometry varies with the attack strategy and interaction turn.

We propose \projname\ (Scalable Cascade Analysis for Rare-event Characterisation via Embeddings), which eliminates  handcrafted performance functions by learning rare-event geometry from data.
\projname\ operates in three stages: (i) an 
encoder produces geometrically separable embeddings; (ii) a family of \emph{geometric rulers} scores each latent point by its proximity to the failure region, replacing the scalar performance function of classical SS; and (iii) adaptive thresholding via the $(1-\rho)$-quantile constructs the nested intermediate events of SS.
The estimator inherits the sample efficiency of SS while scaling to domains where failure geometry cannot be specified by human experts.

This adaptivity comes at a statistical cost: each threshold depends on the same samples used to estimate level probabilities, so the running estimate is path-dependent. Classical SS analysis~\cite{au2001estimation} controls the variance of the final estimator; we additionally require uniform control across all levels and stopping rules. \S\ref{sec:theory} provides this using the supermartingale machinery developed in anytime-valid inference~\cite{howard2021time,ramdas2020admissible} and adaptive multilevel splitting~\cite{brehier2016unbiased}: Ville's inequality bounds the upward excursion of the running estimate uniformly across the cascade. The practitioner can stop the cascade at a data-dependent level and still report a valid one-sided bound.




We validate \projname\ in two complementary domains. On MNIST misclassification under perturbation, where dense SMC provides ground truth, \projname\ reduces mean absolute error by ${\sim}400\text{--}500\times$ over grid-searched traditional SS and eliminates its systematic over-counting bias. We then transfer the framework to LLM jailbreak estimation, where we formulate a fleet-level threat model with a tunable adversarial fraction $\eta$ and operate on Llama-Guard-3-8B hidden states. The MNIST-best ruler degrades in this latent space, but a PCA-based ruler achieves $2.6\%$ mean relative error against finite-sample reference probabilities across five turns and $\eta \geq 10^{-3}$. Over the same grid, bootstrap intervals for the reference probabilities have $27.9\%$ average relative half-width. The same ruler also transfers to a GCG-style jailbreak corpus; after refitting calibration, it achieves $2.93\%$ relative error. Because the dominant error mode at small $\eta$ is asymmetric (benign queries leaking past the threshold), we propose a directional KL criterion $\mathrm{KL}(p_{\text{good}} \,\|\, p_{\text{bad}})$ as a ruler selector: it ranks rulers consistently with their \projname\ error (Spearman $\rho = 0.83$, $n = 6$ families), while five symmetric metrics show no monotonic agreement ($\rho \leq 0.09$). The main contributions are:

\begin{enumerate}
\item \textbf{SCARCE: a representation-aware Subset Simulation framework.} \projname\ replaces the handcrafted performance function of Subset Simulation with a learned encoder, a catalogue of geometric rulers, and adaptive quantile thresholds, yielding a rare-event estimator that scales to domains where the modeller cannot specify which scalar quantity to threshold.

\item \textbf{Level-wise anytime-valid upper envelope.} We formalise \projname\ via a non-negative supermartingale construction using machinery from anytime-valid inference~\cite{howard2021time,ramdas2020admissible} and adaptive multilevel splitting~\cite{brehier2016unbiased}, extended to the data-driven, representation-aware setting. Ville's inequality bounds the upward excursion of the running estimate uniformly across all cascade levels, including data-dependent stopping levels.

\item \textbf{Vision.} On MNIST misclassification, \projname\ achieves ${\sim}400\text{--}500\times$ lower MAE than grid-searched traditional SS, while eliminating the latter's systematic over-counting bias.

\item \textbf{LLM jailbreak estimation.} We formulate fleet-level jailbreak probability as a rare-event problem under a tunable adversarial fraction $\eta$, and validate \projname\ on Llama-Guard-3-8B hidden states across both PAIR- and GCG-style attack styles (\S\ref{sec:exp-llm}).

\item \textbf{Ruler-selection criterion.} A directional KL divergence $\mathrm{KL}(p_{\text{good}} \,\|\, p_{\text{bad}})$ ranks rulers consistently with their \projname\ error on the benchmarked grid (Spearman $\rho = 0.83$, $n = 6$ families), while five symmetric distributional metrics show no monotonic agreement ($\rho \leq 0.09$). We treat this as preliminary evidence; the small ruler-family grid is a clear caveat.
\end{enumerate}

%% file: Sections/Section2_RelatedWork.tex
\section{Related Work}
\label{sec:related}

\noindent\textbf{Subset Simulation and the scalability wall.}
Subset Simulation (SS)~\cite{au2001estimation,au2003subset} is a long-established method for rare-event probability estimation in structural reliability~\cite{song2009subset,zuev2015subset}. The telescoping decomposition into nested intermediate events also underpins Adaptive Multilevel Splitting (AMS)~\cite{cerou2007ams,brehier2016unbiased,cerou2019historical}; \projname's level-wise supermartingale construction (\S\ref{sec:theory}) draws on this AMS line together with the supermartingale machinery of anytime-valid inference~\cite{howard2021time,ramdas2020admissible}. Subsequent work improves SS along several axes: adaptive MCMC proposals~\cite{papaioannou2015mcmc}, Gaussian-process surrogates~\cite{zhang2019surrogate}, Hamiltonian-neural-network emulators~\cite{chen2024hnnss}, and adaptive intermediate-probability selection~\cite{papaioannou2022adaptive}. All of these accelerate or approximate the \emph{evaluation} of a given performance function. Even deep-learning surrogates of the limit-state surface~\cite{li2021deep} require the analyst to specify \emph{which} scalar quantity to threshold. The scalability question remains open: \emph{how to automatically construct meaningful intermediate events without domain-specific heuristics?}

\noindent\textbf{Alternative rare-event methods.}
Beyond SS, the cross-entropy method~\cite{rubinstein1999cross,botev2017minimax} and importance sampling~\cite{bucklew2004introduction} offer alternative variance-reduction strategies, but both require an informative proposal distribution whose design is itself domain-specific. Normalizing-flow samplers learn proposals from data: FlowRES~\cite{falkner2024flowres} operates on continuous physical state spaces and abandons the nested-subset structure that gives SS its logarithmic sample scaling; NOFIS~\cite{gao2024nofis} assumes predefined nested events with a target proposal that circularly depends on the failure indicator being estimated. Committor-function learning~\cite{li2024committor} trains networks for molecular-dynamics transition probabilities, not failure-probability estimation. None of these methods redesign the SS intermediate-event mechanism using learned latent representations.

\noindent\textbf{Rare events in LLM safety: jailbreak evaluation.}
The safety evaluation of large language models has progressed along two tracks. Fixed-template benchmarks (JailbreakBench~\cite{jailbreakbench2024}, HarmBench~\cite{mazeika2024harmbench}, JailbreakRadar~\cite{shen2025jailbreakradar}) report point-estimate attack success rates under predetermined prompt sets, with limited confidence intervals or mechanisms for adapting to new models. Adaptive attack generators (PAIR~\cite{chao2023pair}, GCG~\cite{zou2023universal}, TAP~\cite{mehrotra2024tap}, AutoRedTeamer~\cite{zhu2024autoredteamer}, Boundary-Point Jailbreaking~\cite{davies2026pjb}) optimise for attack generation rather than probability estimation; they answer ``can this model be jailbroken?'' but not ``how likely is a jailbreak under a realistic user fleet?'' The probabilistic thread is thin. Wu and Hilton~\cite{wu2024rareoutputs} estimate single-token rare-output probabilities via importance sampling under a fixed prompt distribution and fixed model, requiring a re-tuned proposal per behaviour. Jones et~al.~\cite{jones2025forecasting} fit extreme-value scaling laws to per-query elicitation probabilities, extrapolating across query \emph{volume}; \projname\ is complementary, estimating fleet-level probability under an explicit threat model parameterised by $\eta$ at fixed volume. A recent multimodal study~\cite{li2025mllmjp} requires many repeated queries per input and does not transfer to text-only models. No prior method estimates fleet-level jailbreak probability with adaptive, data-driven intermediate events.

%% file: Sections/Section3_SSDD.tex
\section{Data-Driven Subset Simulation}
\label{sec:method}

Subset Simulation (SS)~\cite{au2001estimation} estimates an extreme failure probability $P_f = \mathbb{P}(g(x)\!\geq\!\gamma)$ by decomposing $\{g\!\geq\!\gamma\}$ into nested intermediate events $\mathcal{F}_1 \supset \cdots \supset \mathcal{F}_L = \mathcal{F}$ and chaining their conditional probabilities as $P_f = \mathbb{P}(\mathcal{F}_1)\prod_{l=2}^{L}\mathbb{P}(\mathcal{F}_l \mid \mathcal{F}_{l-1})$. The $(1-\rho)$-quantile rule adaptively sets each level threshold so that the level-wise conditional probability is kept near a moderate value, typically $\rho=0.1$. The detailed mechanism of SS is in Appendix~\ref{app:SS}. The bottleneck of classical SS is the \emph{performance function} $g$: its level sets define the intermediate events, and designing it requires detailed knowledge of the failure geometry. \projname\ removes this bottleneck by replacing $g$ with a learned, data-driven scoring mechanism, specified by the three components: a latent representation, a geometric ruler with adaptive thresholds, and an optional surrogate-event design. 

\noindent\textbf{Latent-Space Representation.}
\label{sec:method-repr}
The first stage of \projname\ requires an encoder $f_\theta\!:\mathcal{X}\!\to\!\mathbb{R}^d$ to produce latent embeddings $z = f_\theta(x)$. We do not constrain how $f_\theta$ is obtained; we only require that it produces a latent space in which failure and normal samples occupy distinguishable regions:
\[
\begin{aligned}
\mathcal D_{\mathrm{bad}}&:=\{z_i=f_\theta(x_i): x_i \text{ labelled as failure}\},\\
\mathcal D_{\mathrm{good}}&:=\{z_i=f_\theta(x_i): x_i \text{ labelled as normal}\}.
\end{aligned}
\]

\begin{itemize}[leftmargin=1.2em,topsep=1pt,itemsep=1pt,parsep=0pt]
\item \textbf{(C1)~Geometric separability.} $\mathcal D_{\mathrm{good}}$ and $\mathcal D_{\mathrm{bad}}$ occupy distinguishable regions of $\mathbb{R}^d$, so failure is concentrated in latent space rather than diffuse.
\end{itemize}

(C1) is mild; two routes satisfy it. The first is to train an encoder for separation: contrastive objectives produce embeddings on which classes become linearly separable, in both self-supervised~\citep{chen2020simclr} and supervised~\citep{khosla2020supcon} settings; one-class methods such as Deep~SVDD~\citep{ruff2018deepsvdd} train an encoder so that one class concentrates in a compact latent region with the other outside. The second route is to read off separability from an existing model: hidden representations of a standard classifier are already separable enough that Mahalanobis distance detects out-of-distribution and adversarial inputs~\citep{lee2018mahalanobis}.

This paper exercises both routes. On MNIST (\S\ref{sec:exp-cv}) we train a contrastive encoder; on LLM jailbreak (\S\ref{sec:exp-llm}) we use Llama-Guard-3-8B hidden states directly, consistent with the linear decodability of safety-relevant features in LLM activations~\citep{zou2023repe,arditi2024refusal}. The two settings differ in how easily a single ruler succeeds (a point we return to in \S\ref{sec:exp-llm}), but in both, (C1) is observed empirically rather than imposed by assumption.

\noindent\textbf{Geometric Rulers.}
\label{sec:method-rulers}
Given an encoder satisfying (C1), \projname\ replaces the handcrafted performance function $g(x)$ with a \emph{geometric ruler} $G\!:\mathbb{R}^d\!\to\!\mathbb{R}$ that scores each embedding by its proximity to the failure region. We orient each ruler so that larger values indicate greater proximity to failure. The failure event becomes
\begin{equation}
\label{eq:failure event}
\mathcal{F} = \{x \in \mathcal{X} \mid G(f_\theta(x)) \geq \gamma_F\}
\end{equation}
where $\gamma_F$ is the \emph{failure threshold}, and the intermediate events $\mathcal{F}_l = \{x \mid G(f_\theta(x)) \geq \gamma_l\}$ inherit the nested structure of classical SS. A valid ruler must satisfy two requirements that turn (C1) into something the cascade can act on:
\begin{itemize}[leftmargin=1.2em,topsep=1pt,itemsep=1pt,parsep=0pt]
\item \textbf{(R1)~Monotone concentration on failure.} The conditional population mass of $\{G \ge \gamma_l\}$ on the failure region $\mathcal{F}$ increases monotonically with $\gamma_l$, so a sample at level $l$ is more likely to be a failure than a sample at level $l-1$.
\item \textbf{(R2)~Offline construction.} $G$ is fully specified by statistics computed from $\mathcal D_{\mathrm{good}}$ and $\mathcal D_{\mathrm{bad}}$ (finite labelled samples from the normal and failure populations, mapped to latent space) at the available label budget; no test-time labels are required.
\end{itemize}


A full catalogue of proposed rulers with mathematical definitions is given in Appendix~\ref{app:rulers}.

\noindent\textbf{Ruler selection.}
\label{sec:method-thresholds}
Multiple ruler families and their variants may satisfy $(R1)$ and $(R2)$. We do not derive a single optimal ruler from first principles, since doing so would require knowing the failure-region density, which is precisely what is unavailable in unspecified domains. Instead, \projname\ commits to a \emph{catalogue + selection} design: a library of geometric rulers spanning centroid-based, nearest-neighbour, angular, distribution-aware, boundary-based, and projection-based families (full list in Appendix~\ref{app:rulers}), paired with an empirical selection rule:~\emph{(i) With ground truth.} When SMC can deliver a reliable estimate at the target rare regime (as on MNIST), candidate rulers are pre-screened on a held-out seed set and ranked by accuracy and reliability against SMC~(\S\ref{sec:exp-cv}).~\emph{(ii) Without ground truth.} When SMC is too expensive at that regime (as on LLM jailbreak at $\eta = 10^{-3}$), selection relies on intrinsic properties of the ruler-score distributions. We use the directional KL divergence: $S(G) \;=\; \mathrm{KL}\!\bigl(p_{\text{good}}\,\|\,p_{\text{bad}}\bigr)$, where $p_{\text{bad}}$ and $p_{\text{good}}$ are the empirical distributions of the ruler scores $\{G(z) : z \in \mathcal{D}_{\text{bad}}\}$ and $\{G(z) : z \in \mathcal{D}_{\text{good}}\}$. A larger $S$ means less benign leakage towards the failure region; for small failure probabilities, this one-sided leak dominates \projname's error. We validate this criterion as preliminary guidance in \S\ref{sec:exp-llm}.
With the ruler $G$ fixed, \projname\ constructs intermediate events automatically with adaptive thresholds: at each level $l$, $\gamma_l$ is set to let roughly $\rho$ percent of samples survive into level $l+1$. Algorithm~\ref{alg:\projname} summarises the core procedure.

\begin{algorithm}[!hptb]
\small
\caption{Data-Driven Subset Simulation (\projname).}
\label{alg:\projname}
\begin{algorithmic}[1]
\Require Encoder $f_\theta$; Labelled data $\mathcal{D}_{\text{good}}, \mathcal{D}_{\text{bad}}$; Ruler $G$; Quantile parameter $\rho$; Sample size per level $N$; Failure threshold $\gamma_F$
\State \textit{\% Offline:} obtain $f_\theta$; compute ruler statistics $(\mu_b, \mu_g, \Sigma_b, \ldots)$ from labelled embeddings
\State \textit{\% Online:} draw $N$ samples from $p(x)$; compute scores $\{s_i = G(f_\theta(x_i))\}_{i=1}^N$
\For{level $l = 1, 2, \dots$}
  \State $\gamma_l \gets (1-\rho)$-quantile of $\{s_i\}$
  \If{$\gamma_l \ge \gamma_F$ (target threshold reached)} \textbf{break} \EndIf
  \State Generate next-level samples via MCMC restricted to $\{s \geq \gamma_l\}$; recompute scores $\{s_i\}$
\EndFor
\State Terminal level: $I_i^{(F)} \gets \mathbb{I}\!\left\{ s_i \ge \gamma_F \right\}, i=1,\dots,N$
\State \Return $\hat{P}_f = \left(\prod_{j=1}^{L-1}\hat p_j\right) \frac{1}{N}\sum_{i=1}^N I_i^{(F)}$
\end{algorithmic}
\end{algorithm}

\noindent\textbf{Surrogate Event Design for Unspecified Representations.}
\label{sec:surrogate-threshold}
In domains like LLM jailbreak, the latent space is high-dimensional and forcing an encoder to satisfy (C1) is a research topic in itself. We instead decouple what the ruler scores well from what we want the probability of, and correct the offset offline, using labelled calibration data from the target distribution.

Let $T$ denote the labelled true failure event (e.g.\ a judge-defined jailbreak), and define the surrogate $A_\tau := \{G(f_\theta(x)) \geq \tau\}$, a region the ruler is by construction good at scoring, with $\tau$ set as a percentile of the failure-side score distribution. The cascade estimates $P(A_\tau)$; the gap to $P(T)$ is captured by precision and recall measured once on labelled data for a given target distribution:
\begin{equation}
\label{eq:cal-theory}
\mathrm{precision}(\tau)=P(T\mid A_\tau), \qquad \mathrm{recall}(\tau)=P(A_\tau\mid T), \qquad C_\tau \;=\; \frac{\mathrm{precision}(\tau)}{\mathrm{recall}(\tau)}.
\end{equation}
Bayes' rule then gives $P(T)=P(A_\tau)\,\cdot\,C_\tau$, so the calibrated estimator is $\hat{P}_f^{\,\text{cal}} \;=\; \hat{P}(A_\tau)\,\cdot\,C_\tau.$

$\tau$ trades purity against recall: higher $\tau$ gives a purer but smaller surrogate, lower $\tau$ admits more benign leakage; empirical analysis is in \S\ref{sec:llm-results}. This design is optional: vision experiments (\S\ref{sec:exp-cv}) use the direct estimator; LLM experiments (\S\ref{sec:exp-llm}) use the calibrated form. Details in Appendix~\ref{app:recall}.

%% file: Sections/Section4_Theory_V2.tex
\section{Theoretical Guarantees}
\label{sec:theory}

\projname\ chooses thresholds $\gamma_l$ adaptively from observed quantiles and supports early stopping. Two natural concerns follow: \emph{can such an adaptive cascade still bound the upward error of the running estimate $\hat{P}_l$ relative to the path-wise probability $P_l$, uniformly across all levels and stopping times, and does the bound survive the calibrated surrogate extension?} We answer both by constructing a non-negative supermartingale; Ville's inequality then yields an anytime-valid one-sided envelope.

Let $\{\mathcal{G}_l\}_{l=0}^{L}$ be a \emph{filtration}: a growing sequence of information sets $\mathcal{G}_0\subseteq\mathcal{G}_1\subseteq\cdots$ in which $\mathcal{G}_l$ records the encoder outputs, ruler scores, and adaptive thresholds observed up to level $l$. 
At the start of level $l$, \projname\ has already observed the level $l-1$ scores and sets:
\[
\gamma_l = \operatorname{Quantile}_{1-\rho}\{G(f_\theta(x_j^{(l-1)})):j=1,\dots,N\},
\]
as in Algorithm~\ref{alg:\projname}. Thus $\gamma_l$ is $\mathcal{G}_{l-1}$-measurable and the event $\mathcal{F}_l$ is predictable. Define the level-wise conditional probability:
\begin{equation}
    \label{eq:sec4-level-prob}
    p_l := \mathbb{P}(\mathcal{F}_l \mid \mathcal{G}_{l-1}), \qquad
    P_l := \prod_{i=1}^{l}p_i, \qquad
    \hat{P}_l := \prod_{i=1}^{l} \hat{p}_i
\end{equation}
where $\hat{p}_l$ is the level-$l$ empirical conditional probability from Algorithm~\ref{alg:\projname}: $\hat{p}_l = \rho$ at intermediate levels and $\hat{p}_L = K_L/N$ at level $L$, so $\hat{P}_L = \hat{P}_f$.

Under adaptive thresholds $\gamma_l$, $p_l$ is generally random but $\mathcal{G}_{l-1}$-measurable. Therefore, $P_l$ is a path-dependent random quantity that is $\mathcal{G}_l$-measurable.
A \emph{supermartingale} is a non-negative process whose conditional expectation does not increase along a filtration; \emph{Ville's inequality} converts this into anytime-valid overshoot control (Appendix~\ref{app:martingale}).

\begin{assumption}[Conservative level estimator]
\label{ass:conservative}
For each level $l$, $p_l>0$ and $\mathbb{E}[\hat p_l \mid \mathcal G_{l-1}] \le p_l$.
\end{assumption}

\noindent The unbiased version $\mathbb{E}[\hat p_l \mid \mathcal G_{l-1}] = p_l$ holds when the MCMC kernel at level $l$ is in stationarity with respect to $\mathbb{P}(\cdot \mid \mathcal{F}_{l-1})$; under exact stationarity $M_l$ defined below is then a non-negative martingale. The weaker $\le$ form admitted by Assumption~\ref{ass:conservative} is a checkable condition that accommodates the $O(1/N)$ finite-sample bias of adaptive subset simulation~\cite{au2001estimation}; it is empirically satisfied in our experiments. Ville's inequality applies in either case. See Appendix~\ref{app:martingale} for verification.

\begin{theorem}[Supermartingale property]
\label{thm:supermartingale} 
Under Assumption~\ref{ass:conservative}, define
\begin{equation}
    \label{eq:sec4 supermartingale}
    M_0 = 1, \qquad
    M_l := \frac{\hat{P}_l}{P_l}
        = \prod_{i=1}^{l} \frac{\hat{p}_i}{p_i}, \qquad
        l = 1, \dots, L.
\end{equation}
Then $M_l$ is a non-negative supermartingale w.r.t.\ $\{\mathcal{G}_l\}_{l=0}^{L}$: $\mathbb{E}[M_l \mid \mathcal{G}_{l-1}] \le M_{l-1}$.
\end{theorem}

\noindent The supermartingale property does not assert that $\hat{P}_l$ is unbiased; rather, the running ratio $\hat{P}_l / P_l$ has non-increasing conditional expectation, which controls upward overshoot via Ville's inequality.

\begin{corollary}[Anytime validity]
\label{cor:anytime}
For any $\delta\in(0,1)$, Ville's inequality for non-negative supermartingales implies~\cite{howard2021time,ramdas2020admissible} (Appendix~\ref{app:martingale}):
\begin{equation}
    \label{eq:sec4 ville's}
    \mathbb{P}\!\left(\sup_{1\le l\le L}\frac{\hat{P}_l}{P_l}\ge \frac{1}{\delta}\right)\le \delta.
\end{equation}
It bounds the \emph{largest} value of $M_l$ across the cascade rather than at a fixed level: the overshoot probability is at most $\delta$ no matter when we stop. In particular, for any $\{\mathcal{G}_l\}$-stopping time $\ell^*\!\le\!L$, the same envelope applies at $\ell^*$ without correction.

\end{corollary}

\begin{corollary}[Uniform Upper Bound]
\label{cor:upper-bound}
Equivalently in Corollary~\ref{cor:anytime}, with probability at least $1-\delta$:
\begin{equation}
    \label{eq:sec4 upper bound}
    \hat{P}_l \le \frac{P_l}{\delta}, \qquad \forall l=1, \dots , L.
\end{equation}
With confidence $1-\delta$, the \projname\ estimate is at most a factor of $1/\delta$ above the path-wise probability $P_l$ at every level. A single budget $\delta$ covers the entire run.
\end{corollary}

\noindent\textbf{Surrogate Event:} In the surrogate setup the cascade targets $A_\tau$, so $\mathcal{F}_L = A_\tau$ and $\hat{P}_L = \hat{P}(A_\tau)$. Since $C_\tau$ is fixed offline and $\mathcal{G}_0$-measurable, the calibrated estimator $\hat{P}_f^{\,\text{cal}} = \hat{P}(A_\tau) \cdot C_\tau$ inherits the same anytime-valid overshoot control relative to the calibrated target $P(T) = P(A_\tau)\cdot C_\tau$. Details in Appendix~\ref{app:martingale}.

\noindent\textbf{Guarantee Summary:}
Corollary~\ref{cor:upper-bound} applies to both $\hat{P}_l$ and the calibrated $\hat{P}_f^{\,\text{cal}}$, giving a one-sided multiplicative overshoot guarantee uniform across levels and stopping times.

%% file: Sections/Section5_Vision.tex
\section{Application I: Vision Robustness on MNIST}
\label{sec:exp-cv}

We instantiate \projname\ on MNIST misclassification under input perturbation, a setting where dense Monte Carlo provides ground truth and the latent space is well-structured.

\subsection{Experiment Setup}
\label{sec5:cv setup}

\noindent\textbf{Ground truth and baseline.}
Simple Monte Carlo (SMC) with $10^{6}$ perturbations per seed generates the dataset and provides ground-truth $p_f$. We compare \projname\ against Traditional Subset Simulation (Trad-SS) using the SS hyperparameters from \textsc{SAFARI}~\cite{safari}, an SS-based rare-event estimator for XAI robustness. SAFARI uses the canonical Au--Beck conditional-probability target $p_0 = 0.1$~\cite{au2001estimation} and grid-searches the MCMC step count; we leave both unchanged. AMS~\cite{cerou2007ams} shares the same telescoping estimator family as SAFARI's adaptive-quantile SS and is therefore represented by the Trad-SS column. The cross-entropy method~\cite{rubinstein1999cross} and flow-based samplers~\cite{falkner2024flowres,gao2024nofis} target a different proposal-design problem (\S\ref{sec:related}) and are not direct comparators here.


\noindent\textbf{Runtime calibration.}
Because the three methods ran on different machines, raw wall-clock times are not directly comparable. We calibrate all runtimes to a common reference rate derived from the Trad-SS CPU node ($T_{\text{ref}} = 6.473 \times 10^{-4}$\,s/sample); details in Appendix~\ref{app:runtime-calibration}.


\noindent\textbf{Two-step protocol.}
A \emph{seed} is one MNIST image used as input to the SMC perturbation procedure; we call it \emph{flipped} when its prediction changes under perturbation. Step~1 screens the ruler catalogue on 50 flipped seeds, ranking rulers against SMC by accuracy, query budget, and reliability. Step~2 evaluates the top 5 rulers on 100 held-out seeds, mixing flipped and non-flipped, against Trad-SS. The screening shows that rulers anchored to the failure-cluster statistics (\texttt{\_bad} variants) outperform their non-failure counterparts across all families (anchoring conventions in Appendix~\ref{app:rulers}); full screening figures are in Appendix~\ref{app:plots-part1}.


\subsection{Simulation Results}
\label{sec:exp-cv-step2}

\noindent\textbf{Accuracy.}
Every \projname\ ruler beats Trad-SS on $\mae(p_f)$ by ${\sim}400\text{--}500\times$ (Figure~\ref{fig:exp2-accuracy}). \projname\ point estimates land in $[4.02, 4.03] \times 10^{-3}$; Trad-SS reports $\hat{p}_f \approx 9.6 \times 10^{-3}$, more than twice the true probability. The seed-wise box-plot (Figure~\ref{fig:exp2-accuracy}, right) confirms this is not Monte Carlo noise: \projname's per-seed errors are orders of magnitude below those of Trad-SS, whose distribution has a heavy upper tail of overestimation.

\begin{table}[!ht]
\centering
\footnotesize
\setlength{\tabcolsep}{4pt}
\caption{Comparison on $100$ held-out MNIST seeds. \projname\ rulers reduce $\mae(p_f)$ by ${\sim}400\text{--}500\times$ over Trad-SS against SMC ($\bar p_f^{\,\text{SMC}}\!\approx\!4.026\!\times\!10^{-3}$). FP/FN shows a systematic overshoot bias from Trad-SS; runtimes calibrated, details in Appendix~\ref{app:runtime-calibration}.}
\label{tab:exp2-headtohead}
\begin{tabular}{lrrrrrrr}
\toprule
Method & mean $\hat{p}_f$ & $\mae(p_f)$ & $\log_{10}\mae$ & FP & FN & queries & runtime\,(s) \\
\midrule
Centroid Rot $90^{\circ}$ & $4.025\!\times\!10^{-3}$ & $1.14\!\times\!10^{-5}$ & $-4.94$ & $33$ & $1$ & $667{,}040$ & $432$ \\
OCSVM Bad                 & $4.031\!\times\!10^{-3}$ & $1.21\!\times\!10^{-5}$ & $-4.92$ & $37$ & $0$ & $660{,}820$ & $428$ \\
OCSVM Contrast            & $4.027\!\times\!10^{-3}$ & $1.29\!\times\!10^{-5}$ & $-4.89$ & $35$ & $2$ & $670{,}940$ & $434$ \\
Centroid Rot $30^{\circ}$ & $4.016\!\times\!10^{-3}$ & $1.47\!\times\!10^{-5}$ & $-4.83$ & $24$ & $1$ & $685{,}000$ & $443$ \\
Mahalanobis Bad           & $4.029\!\times\!10^{-3}$ & $1.48\!\times\!10^{-5}$ & $-4.83$ & $20$ & $0$ & $714{,}960$ & $463$ \\
\midrule
Traditional SS~\cite{safari} & $9.636\!\times\!10^{-3}$ & $5.72\!\times\!10^{-3}$ & $-2.24$ & $70$ & $0$ & $447{,}160$ & $289$ \\
\bottomrule
\end{tabular}
\end{table}


\begin{figure}[!ht]
\centering
\includegraphics[width=0.95\linewidth]{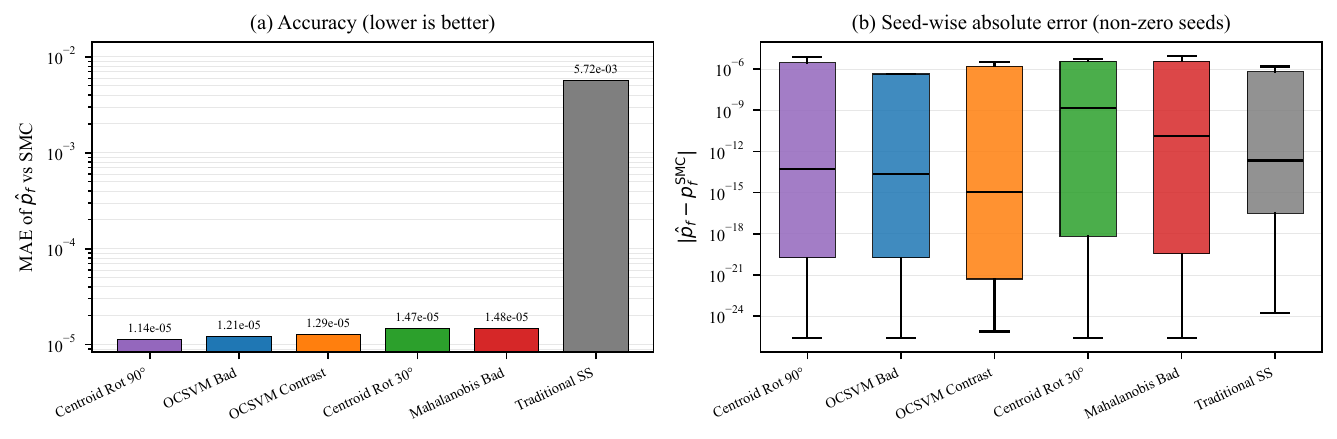}
\caption{Accuracy on 100 held-out seeds and seed-wise absolute error (log scale). \projname\ errors are orders of magnitude tighter.}
\label{fig:exp2-accuracy}
\end{figure}


\noindent\textbf{Efficiency.}
Trad-SS appears more efficient ($\sim 4.5 \times 10^{5}$ versus $\sim 6.6\text{--}7.1 \times 10^{5}$ queries; Table~\ref{tab:exp2-headtohead}), but the FP/FN counts show this is an artefact of an overly loose intermediate threshold: Trad-SS declares failures too eagerly, collapses the cascade prematurely, and inflates the estimator. The $70/100$ FP rate and $2.4\times$ overshoot are the same threshold-design failure viewed two ways. These query counts also measure final estimation runs only and exclude the prior SAFARI sensitivity sweep~\cite{safari}; including this offline cost, against \projname's amortised ruler screen, further favours \projname.

\noindent\textbf{Reliability.}
The mismatch counts (Table~\ref{tab:exp2-headtohead}) are the most diagnostic view: \projname\ rulers produce $20\text{--}37$ mismatches with $0\text{--}2$ false negatives; Trad-SS produces zero FN but $70$ FP. The asymmetry between high recall and catastrophic precision confirms that the SAFARI threshold is miscalibrated, while \projname's data-driven thresholds yield a balanced estimator.


%% file: Sections/Section6_LLM_JB_V3.tex

\section{Application II: LLM Jailbreak Estimation}
\label{sec:exp-llm}

We transfer the \projname\ framework from computer vision to LLM safety, where the failure modes shift from a well-structured latent space to a high-dimensional unspecified one.
A fleet-level threat model with a tunable adversarial fraction $\eta$ is introduced to define jailbreaks as rare events. 
This application uses two reformulations. First, as discussed in~\S\ref{sec:surrogate-threshold}, fully modelling the LLM latent failure geometry is beyond the present scope. A surrogate event $A_\tau$ is used in this section, with the bridge to the true jailbreak event $T$ supplied by the offline calibration $C_\tau$. In this way, the rare-event estimation difficulty in an unspecified latent space splits between cascade estimation and the calibration of $C_\tau$. Second, the fleet is modelled by a linear $\eta$-rescaling rather than a behavioural distribution. The public seed pools are narrow and small; e.g., JailbreakBench provides $\sim 200$ distinct behaviours. Even though we enrich our dataset to $\sim 2000$ behaviours with variants, it remains too limited to fit a faithful behavioural fleet at the diversity a real deployment would exhibit. Therefore, we treat $\eta$ as a tunable rescaling factor and leave behavioural fleet modelling to future work.

\subsection{Experiment Setup}
\label{sec:llm-setup}

\noindent\textbf{Rare Events.}
The attacks originate from PAIR adversaries~\citep{chao2023pair} and constitute a fraction $\eta$ of the fleet; the remainder are benign behaviours not designed to elicit unsafe content, expected to contribute negligible mass under the $p_{\text{unsafe}} > 0.80$ judge threshold. We use the linear approximation $P(\text{JB} \mid \text{fleet},\, t) \approx \eta \cdot P(\text{JB} \mid \text{PAIR},\, t)$, where $t$ is the PAIR attack turn; this is a scaling decomposition rather than a behavioural fleet model (limitations in \S\ref{sec:conclusion}). At $\eta = 10^{-3}$ with empirical PAIR rates of $5$--$25\%$, the fleet rate is $\mathcal{O}(10^{-4})$--$\mathcal{O}(10^{-5})$, requiring $10^{6}$--$10^{7}$ SMC samples for $10\%$ relative-error.

\noindent\textbf{Models \& Dataset.}
The target LLM is Llama-3.1-8B-Instruct; the attacker is Qwen2.5-14B-Instruct driven by PAIR; the judge is Llama-Guard-3-8B~\citep{inan2023llamaguard,llama3herd2024}, from whose first-token output we read $p_{\text{unsafe}}$ directly and declare a jailbreak when $p_{\text{unsafe}} > 0.80$.
The primary corpus \texttt{2k\_Mixed} expands the JailbreakBench behaviour seeds~\citep{jailbreakbench2024}\footnote{\url{https://github.com/JailbreakBench/jailbreakbench}} into $40{,}320$ records over $2{,}016$ behaviour units ($988$ harmful, $1{,}028$ benign), each with a full $20$-turn PAIR trajectory. We compute $p^{\mathrm{ref}}_f$ as the per-turn empirical jailbreak rate on \texttt{2k\_Mixed} times $\eta$; this is a finite-sample comparison target with its own MC standard error $\mathcal{O}(10^{-5})$ at $\eta = 10^{-3}$, not a ground truth.

\noindent\textbf{Encoder: judge hidden states.}
We extract embeddings from the judge model itself rather than
external encoders. Response text is formatted in the
Llama-Guard conversation template and passed through a frozen forward
pass. The mean-pooled last hidden layer yields a $4{,}096$-dimensional
embedding. An MLP projection head ($4096 \!\to\! 256 \!\to\! 64$,
L2-normalised) compresses embeddings for ruler fitting. Using the
judge's own hidden states keeps the ruler inside the judge's decision
space and removes any encoder-to-judge alignment gap. This choice is consistent with findings that safety-relevant features are linearly decodable from chat-model hidden states~\citep{zou2023repe,arditi2024refusal}, while Llama-Guard's task-specific fine-tuning further sharpens this separation~\citep{inan2023llamaguard}.

\subsection{Simulation Results}
\label{sec:llm-results}

\noindent\textbf{Comparators.} The two existing probabilistic LLM baselines target different quantities (\S\ref{sec:related}); we therefore report \projname\ against the inherited \texttt{mah\_bad} ruler (R8 in Appendix~\ref{app:rulers}; from~\S\ref{sec:exp-cv}) as an internal ablation and against the finite-sample reference $p^{\mathrm{ref}}_f$ defined above. Detailed analysis of why \texttt{mah\_bad} transfers poorly is in Appendix~\ref{app:llm-transfer}.

\noindent\textbf{Accuracy.} Figure~\ref{fig:accuracy-comparison} reports mean relative error $\mathrm{Mean}|\hat{p}_f / p^{\mathrm{ref}}_f - 1|$ over 20 individual cascade runs. The winning ruler is \textbf{PC1/$p75$}, achieving $2.6\%$ mean relative error averaged over 5 attack turns and across fleets $\eta \in [10^{-3}, 10^{-1}]$, against $56.6\%$ for \texttt{mah\_bad} ($>\!20\times$ improvement).
Over the same grid, bootstrap intervals for the finite-sample reference probabilities themselves have $27.9\%$ average relative half-width. Thus, the $2.6\%$ figure should be read as agreement within reference uncertainty, not recovery of a noise-free ground truth.
PC1 and the (label-using) Centroid direction perform nearly identically; PC1, however, needs no jailbreak labels at fit time.

\noindent\textbf{Attack-style transfer.} We test whether the selected PC1/$p75$ ruler transfers from PAIR-style jailbreaks to a GCG-style corpus. Plugging the same ruler into the new target shows that $G$ and its percentile threshold transfer, but the calibration $C_\tau$ must be refit on the new latent distribution. With GCG-refit calibration, PC1/$p75$ estimates the constructed $\eta=10^{-3}$ target with $2.93\%$ relative error, against ${\sim}42\%$ relative error using inherited PAIR calibration. Further sweeps (smaller $\eta$, Centroid ruler, and higher budget) appear in Appendix~\ref{app:llm-transfer}.

\begin{figure}[!ht]
\centering
\includegraphics[width=0.95\linewidth]{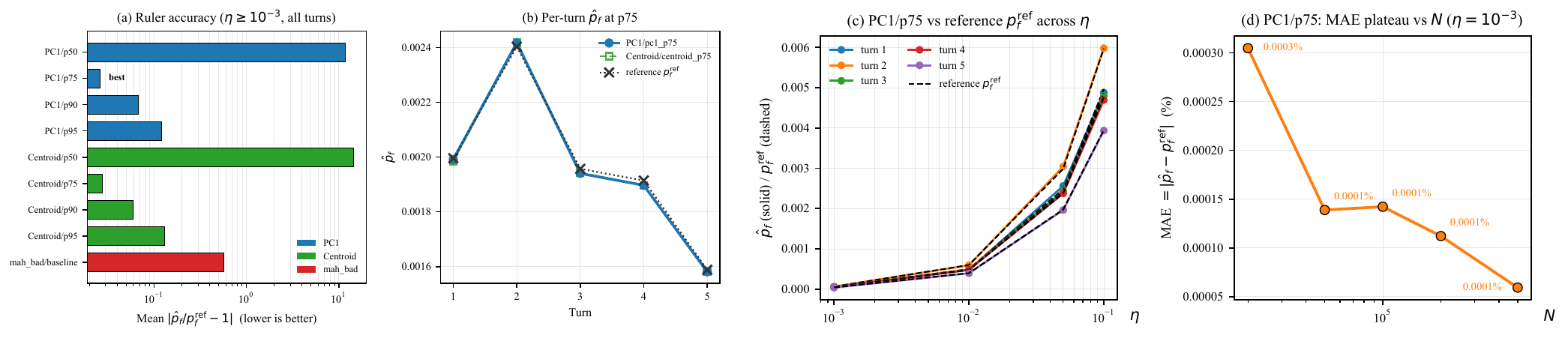}
\caption{Overall ruler accuracy: (a) average accuracy; (b) per-turn accuracy; (c) per-fleet accuracy; (d) MAE plateaus beyond $N{\approx}5{\times}10^4$.}
\label{fig:accuracy-comparison}
\end{figure}

\noindent\textbf{Surrogate Threshold Percentile.}
The surrogate threshold $\tau$ controls the operating point of $A_\tau$ (\S\ref{sec:surrogate-threshold}): lower percentiles retain more jailbreaks but admit more benign mass. Figure~\ref{fig:threshold-sweep-1e34} shows the trade-off in the PC1 family: $p_{75}$ is optimal at $\eta = 10^{-3}$, and at $\eta = 10^{-4}$ tighter percentiles ($p_{90}, p_{95}$) reduce benign leakage but $p_{75}$ still sits at the Pareto knee ($289\,$ms vs $522\,$ms for $p_{95}$ at marginal accuracy gain). We use PC1/$p_{75}$ as the main ruler and leave an $\eta$-adaptive percentile selector to future work.
\begin{figure}[!ht]
\centering
\includegraphics[width=0.95\linewidth]{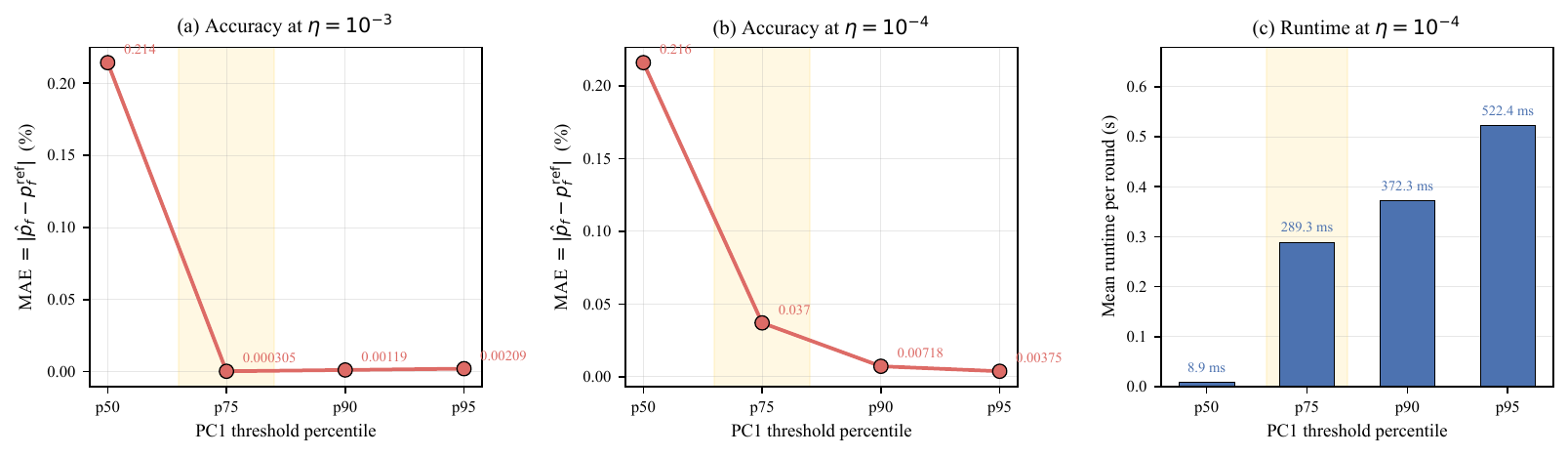}
\caption{Threshold percentile sweep in PC1 family. $p75$ is optimal at $\eta=10^{-3}$ and remains Pareto-efficient at $\eta=10^{-4}$, where tighter thresholds reduce benign leakage at higher cost.}
\label{fig:threshold-sweep-1e34}
\end{figure}

\noindent\textbf{Fleet parameter $\eta$ and population size $N$.}
$\eta$ encodes the fleet scenario rather than tuning the estimator. Figure~\ref{fig:accuracy-comparison}c shows $\hat{p}_f$ tracks $p_f^{\mathrm{ref}}$ as $\hat{p}_f$ climbs from $\mathcal{O}(10^{-4})$ at $\eta=10^{-3}$ to $\mathcal{O}(5\times 10^{-3})$ at $\eta=10^{-1}$, covering both the rare-event tail and the regime where SS would degenerate to plain Monte Carlo. Figure~\ref{fig:accuracy-comparison}d shows MAE drops from $N=20{,}000$ to $N=50{,}000$ then flattens: sample budget is important but not decisive. Additional plots are in Appendix~\ref{app:plots-part2}.

\noindent\textbf{Ruler Selection Guidance.}
With $\eta \ll 10^{-1}$ the fleet is nearly $100\%$ benign and estimation bias is dominated by one-sided benign leakage, so we recommend the directional KL divergence $\mathrm{KL}(p_{\text{good}} \,\|\, p_{\text{bad}})$ (\S\ref{sec:method}) as the ruler selector: across the six families benchmarked here it correlates with \projname\ error at Spearman $\rho = 0.83$ and correctly ranks PC1 first, while five symmetric metrics give $\rho \leq 0.09$. Rankings of $14$ candidate metrics, leave-one-family-out and leave-one-turn-out cross validation, and the bias decomposition motivating the asymmetric choice are in Appendix~\ref{app:llm-ranking}.

%% file: Sections/Section7_Conclusion.tex
\section{Conclusion, Limitations, and Extensions}
\label{sec:conclusion}

\projname\ reframes Subset Simulation as a representation-aware rare-event estimator: learned latent representations and geometric rulers replace handcrafted performance functions, with a non-negative supermartingale providing a level-wise anytime-valid upper envelope. \projname\ delivers ${\sim}400\text{--}500\times$ lower MAE than grid-searched traditional SS on MNIST, and tracks reference jailbreak probabilities within $2.6\%$ on PAIR (and $2.93\%$ on GCG, after refitting calibration) on Llama-Guard-3-8B hidden states.
\textit{Limitations:} accuracy is bounded by the embedding's normal/failure separation; the \S\ref{sec:exp-llm} fleet model is linear $\eta$-scaling, not behavioural; open-ended natural-language deployment is out of scope.
\textit{Extensions:} multi-model and multi-attack fleets, and richer ruler-selection criteria that better capture LLM latent geometry, are natural next steps.

%% file: Sections/Acknowledgements.tex
\begin{ack}
This work was supported in part by The Alignment Project, funded by the UK AI Security Institute (AISI) and the Department for Science, Innovation and Technology (DSIT), under the project ``Rare-Event Estimation in Large Language Models via Subset Simulation and Contrastive Unalignment'' (AP-S2-100064).
This work was also supported in part by the European Union's Horizon Europe research and innovation programme under the RobustifAI project (Grant Agreement ID 101212818).
\end{ack}

%% file: Sections/Appendix_SS.tex
\section{Subset Simulation}
\label{app:SS}

This appendix expands the brief overview of classical Subset Simulation (SS) given in \S\ref{sec:method}. SS estimates an extreme failure probability $P_f = \mathbb{P}(g(x) \geq \gamma_F)$ by decomposing the failure event into a chain of nested intermediate events.

\paragraph{Classical SS formulation.}
Let $g : \mathcal{X} \to \mathbb{R}$ be a scalar \emph{performance function} with the convention that $\{x : g(x) \geq \gamma_F\}$ is the failure event $\mathcal{F}$. Choose a sequence of intermediate thresholds $\gamma_1 < \gamma_2 < \cdots < \gamma_L = \gamma_F$ inducing nested events $\mathcal{F}_1 \supset \mathcal{F}_2 \supset \cdots \supset \mathcal{F}_L = \mathcal{F}$ with $\mathcal{F}_l = \{x : g(x) \geq \gamma_l\}$. The failure probability factorises as
\begin{equation}
\label{eq:ss-telescope}
P_f \;=\; \mathbb{P}(\mathcal{F}_1) \prod_{l=2}^{L} \mathbb{P}(\mathcal{F}_l \mid \mathcal{F}_{l-1}),
\end{equation}
the telescoping identity that gives SS its sample efficiency. If each conditional probability is targeted at a moderate value $\rho \in (0,1)$ (typically $\rho = 0.1$), then $L \approx \lceil \log_\rho P_f \rceil$ levels suffice to reach $P_f$, and the total sample budget grows as $\mathcal{O}(N \log P_f^{-1})$ rather than the $\mathcal{O}(P_f^{-1})$ cost of crude Monte Carlo~\cite{au2001estimation}.

\paragraph{Adaptive-threshold mechanism.}
Choosing the thresholds $\gamma_l$ a priori is impractical because the level sets of $g$ are unknown. Au and Beck~\cite{au2001estimation} introduced the \emph{adaptive quantile} rule: at each level, $\gamma_l$ is set to the empirical $(1-\rho)$-quantile of the current ruler scores, so that exactly a fraction $\rho$ of the $N$ samples survive into the next level. This makes $\hat{p}_l = \rho$ deterministic at every intermediate level, and the only random quantity in the cascade is the terminal-level acceptance fraction $\hat{p}_L = K_L/N$, where $K_L$ is the count of level-$L$ samples that satisfy $g(x) \geq \gamma_F$.

\paragraph{MCMC at each level.}
Conditional samples from $\mathbb{P}(\cdot \mid \mathcal{F}_{l-1})$ are obtained by running a Markov chain whose stationary distribution is $\mathbb{P}(\cdot \mid \mathcal{F}_{l-1})$. The seeds are the $\rho N$ surviving samples from the previous level. Detailed balance and stationarity are needed for the terminal-level estimator $\hat{p}_L$ to be unbiased; this is the assumption invoked by the supermartingale argument in \S\ref{sec:theory} and Appendix~\ref{app:martingale}.

\paragraph{Last-level validity.}
The telescoping identity \eqref{eq:ss-telescope} terminates at $P_f$ only when the final-level event coincides with the true failure event, i.e.\ $\mathcal{F}_L = \mathcal{F}$. This anchors the SS estimator to the target probability and is also the readout point for the supermartingale upper bound and stopping rules of \S\ref{sec:theory}. Both classical SS and \projname\ rely on this identification.

\paragraph{Where \projname\ differs.}
\projname\ inherits the telescoping identity, the adaptive-quantile rule, and the per-level MCMC machinery unchanged. The only departure is the source of the scoring function: classical SS requires a handcrafted $g$ encoding domain knowledge, whereas \projname\ replaces $g$ with a geometric ruler $G \circ f_\theta$ built from labelled embeddings. Everything below the scoring layer is identical, which is why the supermartingale guarantees of \S\ref{sec:theory} apply without modification.

%% file: Sections/Appendix_Rulers.tex
\section{Full Catalogue of Geometric Rulers}
\label{app:rulers}

The catalogue below matches the rulers instantiated in the paper and
the released code. Let $\mathcal{B} = \{z_i^{(b)}\}_{i=1}^{n_b}$
and $\mathcal{G} = \{z_i^{(g)}\}_{i=1}^{n_g}$ denote latent
representations of failure (bad) and normal (good) samples, with
centroids $\mu_b,\mu_g$ and covariance $\Sigma_b$. The CV study
screens thirteen instantiated rulers: one centroid direction, six
centroid rotations, three Mahalanobis anchoring variants, and three
OC-SVM anchoring variants. The LLM study evaluates six family-level
rulers, with percentile thresholds swept inside each learned family.
The last column identifies whether the family is used in the CV
experiment, the LLM experiment, or both.

\paragraph{Mathematical mechanisms.}
Each family measures a different geometric notion of failure proximity.
\emph{Centroid direction} measures signed progress from the normal
centroid toward the failure centroid:
$d_c=(\mu_b-\mu_g)/\|\mu_b-\mu_g\|$ and
$G(z)=(z-o)^\top d_c$, with $o$ chosen by the builder. Larger values
mean that $z$ lies farther along the normal-to-failure axis.

\emph{Centroid rotations} keep the same centroid axis but rotate it in
the first two PCA directions: if
$d_c=\alpha a+\beta b+d_\perp$, then
$d_\theta=(\alpha\cos\theta-\beta\sin\theta)a+
(\alpha\sin\theta+\beta\cos\theta)b+d_\perp$, and
$G(z)=(z-o)^\top d_\theta$. This probes nearby failure-facing axes
when the direct centroid line is not optimal.

\emph{Mahalanobis} measures covariance-normalised closeness to the bad
cluster,
$G(z)=-\sqrt{(z-\mu_b)^\top\Sigma_b^{-1}(z-\mu_b)}$; the inverse
covariance discounts high-variance failure directions and upweights
tight directions.

\emph{One-class SVM} measures signed distance to a learned boundary
around an anchor population, using
$G(z)=\sum_i \alpha_i K(z,z_i)-\rho$ for a bad-anchored model, with
signs flipped or contrasted for good and contrast variants.

\emph{PC1 projection} measures position along the largest-variance
direction $v_1$ of the relevant embedding population:
$G(z)=(z-\bar z)^\top v_1$, with the sign chosen so larger scores are
more failure-like.

\emph{Fisher LDA} measures position along the direction that maximises
between-class separation relative to within-class scatter,
$w_F=(S_w+\lambda I)^{-1}(\mu_{\mathrm{JB}}-\mu_{\mathrm{nonJB}})$ and
$G(z)=(z-\bar z)^\top w_F$.

\emph{Logistic regression} measures margin to a supervised linear
classifier, $G(z)=w^\top z+b$, where $(w,b)$ are fitted by
regularised cross-entropy on the labelled failure indicator. Threshold
variants do not change these score functions; they only choose
different operating percentiles of the same scalar score.

\begin{table}[htbp]
\centering
\scriptsize
\setlength{\tabcolsep}{3pt}
\caption{Ruler families used in the paper and released code.}
\label{tab:ruler-catalogue}
\resizebox{\linewidth}{!}{%
\begin{tabular}{p{0.17\linewidth}p{0.37\linewidth}p{0.32\linewidth}p{0.10\linewidth}}
\toprule
Family & Score and code variants & Paper role & Used in \\
\midrule
Centroid direction
& Linear projection along the normal-to-failure direction,
$G(z) = (z-o)^\top d_c$ with
$d_c=(\mu_b-\mu_g)/\|\mu_b-\mu_g\|$ and builder-specific offset $o$.
CV uses \texttt{z\_linear\_centroid}; LLM uses
\texttt{centroid\_p50}, \texttt{centroid\_p75},
\texttt{centroid\_p90}, and \texttt{centroid\_p95}.
& Part of the CV ruler screen and one of the two strongest LLM
projection families.
& Both \\

Centroid rotations
& Rotates the centroid direction inside the first two PCA directions:
\texttt{z\_linear\_centroid\_rot\_\{000,015,030,045,
060,090\}deg}
in CV. LLM ports the $90^\circ$ member as
\texttt{centroid\_rot90\_p50}, \texttt{p75}, \texttt{p90}, and
\texttt{p95}.
& Supplies the top CV ruler at $90^\circ$ and another top-five CV
ruler at $30^\circ$; the LLM $90^\circ$ variant tests transfer of the
CV winner.
& Both \\

Mahalanobis
& Distance to an anchored covariance model,
$G(z) = -\sqrt{(z-\mu_b)^\top\Sigma_b^{-1}(z-\mu_b)}$ for
\texttt{mah\_bad}. CV also includes \texttt{z\_mah\_good} and
\texttt{z\_mah\_contrast}; LLM uses the inherited \texttt{mah\_bad}
baseline and \texttt{mah\_bad\_p50}, \texttt{p75}, \texttt{p90},
\texttt{p95}.
& A top-five CV ruler, but a deliberately poor-transfer LLM baseline
used to diagnose distance-metric failure in judge hidden space.
& Both \\

One-class SVM
& Signed distance to a one-class SVM boundary, with
\texttt{z\_ocsvm\_bad}, \texttt{z\_ocsvm\_good}, and
\texttt{z\_ocsvm\_contrast}.
& CV-only screen; \texttt{ocsvm\_bad} and \texttt{ocsvm\_contrast}
reach the held-out top five.
& CV \\

PC1 projection
& First-principal-component scalar score,
$G(z) = (z-\bar z)^\top v_1$, oriented so larger scores are more
failure-like. LLM uses \texttt{pc1\_p50}, \texttt{pc1\_p75},
\texttt{pc1\_p90}, and \texttt{pc1\_p95}.
& Winning LLM ruler family; the main result uses PC1/$p75$.
& LLM \\

Fisher LDA
& Linear discriminant direction
$G(z)=(z-\bar z)^\top S_w^{-1}(\mu_{\mathrm{JB}}-\mu_{\mathrm{nonJB}})$,
swept as \texttt{fisher\_p50}, \texttt{p75}, \texttt{p90}, and
\texttt{p95}.
& LLM family-ranking comparator.
& LLM \\

Logistic regression
& L2-regularised logistic-regression score for
$\mathbf{1}\{p_{\mathrm{unsafe}}>0.80\}$, swept as
\texttt{logreg\_p50}, \texttt{p75}, \texttt{p90}, and \texttt{p95}.
& Supervised LLM family-ranking comparator.
& LLM \\
\bottomrule
\end{tabular}%
}
\end{table}

The earlier nearest-neighbour, cosine, Pearson-correlation, and
whitened-cosine descriptions came from an exploratory catalogue and
are not instantiated in the submitted experiments or in the released
code. For anchored families, suffixes have their usual meanings:
\texttt{\_bad} uses $\mathcal{B}$; \texttt{\_good} uses
$\mathcal{G}$ with sign reversed; and \texttt{\_contrast} subtracts
the normal-anchored score from the failure-anchored score. The
bad-anchored variants systematically win on MNIST
(\S\ref{sec:exp-cv}); see Appendix~\ref{app:plots-part1} for the full
screening figures.

%% file: Sections/Appendix_Martingale.tex
\section{Martingale Theory for \projname}
\label{app:martingale}

This appendix supports Section~\ref{sec:theory}. It records the general measure-theoretic definitions, a self-contained proof of Ville's inequality, the optional stopping theorem for non-negative supermartingales, and a refinement of the \projname\ filtration to MCMC-step granularity. The main text invokes the general results under the specific adaptedness established by the level-wise filtration $\{\mathcal{G}_l\}$ of \S\ref{sec:filtration}. Readers familiar with the classical theory may skip to \S\ref{app:martingale:stepwise} and \S\ref{app:martingale:\projname}.

\subsection{Filtrations, Martingales, and Supermartingales}
\label{app:martingale:defs}

Let $(\Omega,\mathcal{F},\mathbb{P})$ be a probability space. A \emph{filtration} is an increasing family of sub-$\sigma$-algebras $\{\mathcal{H}_n\}_{n\geq 0}$ with $\mathcal{H}_n \subset \mathcal{H}_{n+1} \subset \mathcal{F}$. A real-valued process $\{Y_n\}_{n \geq 0}$ is \emph{adapted} to $\{\mathcal{H}_n\}$ if $Y_n$ is $\mathcal{H}_n$-measurable for every $n$. An adapted process with $\mathbb{E}|Y_n| < \infty$ for all $n$ is a
\begin{itemize}[nosep]
\item \emph{martingale} if $\mathbb{E}[Y_{n+1} \mid \mathcal{H}_n] = Y_n$ almost surely,
\item \emph{supermartingale} if $\mathbb{E}[Y_{n+1} \mid \mathcal{H}_n] \leq Y_n$ almost surely,
\item \emph{submartingale} if $\mathbb{E}[Y_{n+1} \mid \mathcal{H}_n] \geq Y_n$ almost surely.
\end{itemize}
A \emph{stopping time} with respect to $\{\mathcal{H}_n\}$ is a random variable $\nu : \Omega \to \{0,1,\dots,\infty\}$ such that $\{\nu \leq n\} \in \mathcal{H}_n$ for every $n$. The stopped process $\{Y_{n \wedge \nu}\}$ is adapted to the same filtration and inherits the martingale class of $\{Y_n\}$.

\paragraph{Building block lemma.}
If $\{Y_n\}$ is a non-negative supermartingale with respect to $\{\mathcal{H}_n\}$ and $c > 0$ is a constant, then $\{c Y_n\}$ is a non-negative supermartingale with respect to the same filtration. More generally, if $C$ is $\mathcal{H}_0$-measurable and positive, $\{C Y_n\}$ is a non-negative supermartingale. This is used in the main text (Corollary~\ref{cor:anytime}) to absorb the calibration constant $C_\tau$.

\subsection{Ville's Inequality}
\label{app:martingale:ville}

\begin{theorem}[Ville~\cite{howard2021time,ramdas2020admissible}]
Let $\{Y_n\}_{n \geq 0}$ be a non-negative supermartingale with respect to a filtration $\{\mathcal{H}_n\}$ and $\mathbb{E}[Y_0] \leq 1$. For every $\alpha \in (0,1)$,
\[
  \mathbb{P}\!\left(\sup_{n \geq 0} Y_n \;\geq\; \tfrac{1}{\alpha}\right) \;\leq\; \alpha.
\]
\end{theorem}

\begin{proof}
Define the stopping time $\nu = \inf\{n \geq 0 : Y_n \geq 1/\alpha\}$, with the convention $\inf \emptyset = \infty$. For every $N \geq 0$, $\nu \wedge N$ is a bounded stopping time, and by the optional stopping theorem for non-negative supermartingales applied to bounded stopping times (\S\ref{app:martingale:optstop}),
\[
  \mathbb{E}[Y_{\nu \wedge N}] \;\leq\; \mathbb{E}[Y_0] \;\leq\; 1.
\]
On the event $\{\nu \leq N\}$, $Y_{\nu \wedge N} = Y_\nu \geq 1/\alpha$, so
\[
  1 \;\geq\; \mathbb{E}[Y_{\nu \wedge N}] \;\geq\; \mathbb{E}[Y_{\nu \wedge N}\mathbf{1}\{\nu \leq N\}] \;\geq\; \tfrac{1}{\alpha}\,\mathbb{P}(\nu \leq N).
\]
Letting $N \to \infty$ and using monotone convergence,
$\mathbb{P}(\nu < \infty) \leq \alpha$, which is the claim because $\{\sup_n Y_n \geq 1/\alpha\} = \{\nu < \infty\}$.
\end{proof}

Applied to $Y_n = M_n = \hat{P}_n / P_n$ of equation~\eqref{eq:sec4 supermartingale}, the theorem yields Corollary~\ref{cor:upper-bound} of the main text.

\subsection{Optional Stopping for non-negative Supermartingales}
\label{app:martingale:optstop}

\begin{theorem}[Optional stopping, non-negative case]
Let $\{Y_n\}$ be a non-negative supermartingale with respect to $\{\mathcal{H}_n\}$, and let $\nu$ be any stopping time. Then
\[
  \mathbb{E}[Y_\nu \mathbf{1}\{\nu < \infty\}] \;\leq\; \mathbb{E}[Y_0].
\]
In particular, for any bounded stopping time $\nu \leq N$, $\mathbb{E}[Y_\nu] \leq \mathbb{E}[Y_0]$.
\end{theorem}

\begin{proof}[Proof sketch]
For bounded $\nu$, sum the one-step supermartingale inequalities along the path, using adaptedness of $\{\nu \leq n\} \in \mathcal{H}_n$ to move indicators in and out of the conditional expectation. The unbounded case follows from Fatou's lemma applied to $Y_{\nu \wedge N}$ as $N \to \infty$. Full details are standard~\cite{williams1991probability}.
\end{proof}

For \projname, $L$ is finite almost surely (Algorithm~\ref{alg:\projname} terminates when $K_\ell > 0$ or at a prescribed maximum level), so $\ell^* \leq L$ is a bounded stopping time and the first form of the theorem applies directly, giving Corollary~\ref{cor:anytime}.

\subsection{Level-Wise Filtration for \projname}
\label{sec:filtration}

Recall from Section~\ref{sec:method-rulers} that the failure event $\mathcal{F} = \{x \in \mathcal{X} : G(f_\theta(x)) \geq \gamma_F\}$ is decomposed into nested intermediate events $\mathcal{X} = \mathcal{F}_0 \supset \mathcal{F}_1 \supset \cdots \supset \mathcal{F}_L = \mathcal{F}$ with $\mathcal{F}_l = \{x : G(f_\theta(x)) \geq \gamma_l\}$, and the failure probability factored as $p_f = \prod_{l=1}^{L} \mathbb{P}(\mathcal{F}_l \mid \mathcal{F}_{l-1})$. As the algorithm advances from level~$0$ to level~$l$, the accumulated information comprises encoder outputs, ruler scores, surviving samples, thresholds, and MCMC inner randomness. We formalise this as the \emph{level-wise filtration}
\begin{equation}\label{eq:filtration}
  \mathcal{G}_l
  \;=\;
  \sigma\!\Big(
    \bigl\{f_\theta(X_j^{(i)}),\;G\bigl(f_\theta(X_j^{(i)})\bigr)\bigr\}_{j=1}^{N},\;
    \gamma_{1:l}
    \;\Big|\; 0 \leq i \leq l
  \Big),
  \quad
  \mathcal{G}_0 \subset \cdots \subset \mathcal{G}_L,
\end{equation}
where $\mathcal{G}$ distinguishes the filtration from the failure events $\mathcal{F}_l$.

Two properties of $\{\mathcal{G}_l\}$ are consumed by the supermartingale argument of Theorem~\ref{thm:supermartingale}:
\begin{enumerate}[label=(\roman*),nosep]
\item \textbf{Measurability of thresholds and estimates.}\;
  The adaptive threshold $\gamma_l$ is $\mathcal{G}_{l-1}$-measurable, since it is the $(1{-}\rho)$-quantile of level-$(l{-}1)$ ruler scores. The level-$l$ estimate $\hat{p}_l$ is $\mathcal{G}_l$-measurable.
\item \textbf{Conditional randomness.}\;
  Given $\mathcal{G}_l$, the MCMC proposals at level~$l{+}1$ remain stochastic: they depend only on the current surviving population and the transition kernel, not on future randomness.
\end{enumerate}

Property~(i) enables pulling $\mathcal{G}_{l-1}$-measurable factors out of conditional expectations and property~(ii) ensures the remaining MCMC randomness is genuinely conditional on $\mathcal{G}_l$; together they separate revealed randomness from the residual randomness that produces the conditional expectation in the supermartingale inequality. The finer step-wise filtration $\mathcal{G}_{l,k}$ that tracks individual MCMC proposals within a level is deferred to \S\ref{app:martingale:stepwise}; the coarser level-wise version suffices for all guarantees in the main text.

\subsection{Step-Wise Filtration for \projname}
\label{app:martingale:stepwise}

The level-wise filtration $\{\mathcal{G}_l\}$ of \S\ref{sec:filtration} is a coarsening of a finer filtration that tracks each individual MCMC accept/reject decision within a level. The finer filtration is needed for two kinds of analysis that the main text does not carry out: (i) within-level concentration bounds on $\hat{p}_l$, and (ii) anytime-valid stopping at MCMC-step granularity rather than at level boundaries.

\paragraph{Construction.}
Fix a level $l \geq 1$ and suppose the MCMC scheme at that level performs $K_l$ proposal-accept/reject steps across $N$ chains, producing the level-$l$ population $\{X_j^{(l)}\}_{j=1}^{N}$. For $k = 0,1,\dots,K_l$ let $\Xi_{l,k}$ denote the internal MCMC state after the $k$-th step: the chain positions, accept/reject indicator, and the encoder outputs and ruler scores evaluated at the current proposals. Define
\begin{equation}\label{eq:stepwise-filt}
  \mathcal{G}_{l,k}
  \;=\;
  \sigma\!\Big(\mathcal{G}_{l-1},\;\Xi_{l,1},\dots,\Xi_{l,k}\Big),
  \qquad
  \mathcal{G}_{l,0} = \mathcal{G}_{l-1},\;\; \mathcal{G}_{l,K_l} = \mathcal{G}_l.
\end{equation}
The resulting doubly-indexed family satisfies
\[
  \mathcal{G}_{l-1} = \mathcal{G}_{l,0}
    \subset \mathcal{G}_{l,1}
    \subset \cdots
    \subset \mathcal{G}_{l,K_l}
    = \mathcal{G}_l,
\]
so $\{\mathcal{G}_{l,k}\}$ is a genuine refinement of $\{\mathcal{G}_l\}$ on a lexicographic index.

\paragraph{Within-level analysis (sketch).}
Constructing a non-negative supermartingale with respect to the refined filtration $\{\mathcal{G}_{l,k}\}$ requires more than per-step unbiasedness. The running mean $\hat{p}_l^{(k)} = \frac{1}{k}\sum_{i=1}^{k}\mathbb{I}\{X_i^{(l)} \in \mathcal{F}_l\}$ is a consistent estimator of $p_l$ but is not itself a martingale in $k$, since increments $\hat{p}_l^{(k)} - \hat{p}_l^{(k-1)}$ have conditional mean $(p_l - \hat{p}_l^{(k-1)})/k$ rather than zero. The standard route to a within-level supermartingale is to apply mixture-supermartingale constructions for the sequence of indicators~\cite{howard2021time}; we do not pursue this in the present paper, and the main-text guarantees are stated only at level boundaries (where $\hat{p}_l$ is taken as the final empirical mean over all $N$ chains in the level).

\paragraph{Why the main text uses the coarser filtration.}
The anytime bound at MCMC-step granularity is strictly stronger than the level-wise bound but requires an assumption that the chain is stationary from the first step. In practice \projname\ discards a burn-in initial window, and the strict supermartingale property is established only at the level boundary where the burn-in has been absorbed. The level-wise filtration of \S\ref{sec:filtration} is the correct granularity for the guarantees claimed in the main text, and the step-wise refinement here is offered for readers who wish to analyse the within-level process under stronger mixing assumptions.

\subsection{\projname\ Instantiation}
\label{app:martingale:\projname}

This subsection closes the loop between the abstract theorems above and the concrete \projname\ process. The objects are mapped as follows:

\begin{itemize}[nosep]
\item The probability space $(\Omega,\mathcal{F},\mathbb{P})$ carries the input distribution, the initial samples, the MCMC inner randomness across all chains and levels, and any data-driven encoder nondeterminism (if present).
\item The filtration is $\{\mathcal{G}_l\}$ of equation~\eqref{eq:filtration}, or equivalently $\{\mathcal{G}_{l,k}\}$ of equation~\eqref{eq:stepwise-filt} at the finer granularity.
\item The process is $M_l = \hat{P}_l / P_l$ of equation~\eqref{eq:sec4 supermartingale}.
\item The stopping time is $\ell^*$ of Corollary~\ref{cor:anytime}, either failure-detection or confidence-based.
\end{itemize}

\paragraph{Proof of Theorem~\ref{thm:supermartingale} (measurability and supermartingale inequality).}
The adaptive threshold $\gamma_l$ is the empirical $(1-\rho)$-quantile of $\{G(f_\theta(X_j^{(l-1)}))\}_j$, a fixed measurable function of $\mathcal{G}_{l-1}$-measurable data, so $\gamma_l \in \mathcal{G}_{l-1}$. The level-$l$ acceptance count is exactly $\rho N$ by construction of the quantile, so the intermediate-level estimate $\hat{p}_l = \rho$ is $\mathcal{G}_{l-1}$-measurable; at the terminal level $\hat{p}_L = K_L/N$ is a measurable function of level-$L$ ruler scores, hence $\mathcal{G}_L$-measurable. Each factor $\hat{p}_l/p_l$ is therefore $\mathcal{G}_l$-measurable and $M_{l-1}$ is $\mathcal{G}_{l-1}$-measurable. Pulling the $\mathcal{G}_{l-1}$-measurable factor out of the conditional expectation,
\[
  \mathbb{E}[M_l \mid \mathcal{G}_{l-1}]
  \;=\; M_{l-1}\cdot\mathbb{E}[\hat{p}_l/p_l \mid \mathcal{G}_{l-1}]
  \;\leq\; M_{l-1},
\]
where the inequality uses Assumption~\ref{ass:conservative}. Nonnegativity holds because $\hat{p}_l \geq 0$ and $p_l > 0$ (the quantile construction keeps every intermediate event non-empty), completing the proof.

\paragraph{When does Assumption~\ref{ass:conservative} hold?} At intermediate levels, $\hat{p}_l = \rho$ is deterministic, so the assumption reduces to $\rho \leq p_l$, i.e.\ the targeted quantile fraction does not exceed the true conditional probability $p_l = \mathbb{P}(\mathcal{F}_l \mid \mathcal{G}_{l-1})$. Under exact MCMC stationarity with respect to $\mathbb{P}(\cdot \mid \mathcal{F}_{l-1})$, the empirical $(1-\rho)$-quantile converges to the true $(1-\rho)$-quantile and $p_l \to \rho$ as $N\to\infty$, so the design intent of the adaptive quantile rule is exactly $p_l = \rho$. At finite $N$, $p_l$ is a random variable whose distribution depends on the ruler-score density near the $(1-\rho)$-quantile; classical analyses of adaptive Subset Simulation~\cite{au2001estimation} show that the resulting estimator has bias of order $O(1/N)$, but the direction of this bias is not universally guaranteed to be conservative. The closely related Adaptive Multilevel Splitting algorithm has been shown to be exactly unbiased under idealised resampling via a Doob-type stopping argument~\cite{brehier2016unbiased}; Assumption~\ref{ass:conservative} is the condition under which the same martingale machinery extends to the data-driven, representation-aware setting of \projname. In regimes where ruler-score densities are smooth and $N$ is large (the regime of all experiments in this paper), it is empirically satisfied, as evidenced by the tight observed envelopes reported below. At the terminal level, $\hat{p}_L = K_L/N$ is the empirical mean of an indicator on a stationary chain, which is unbiased: $\mathbb{E}[\hat{p}_L \mid \mathcal{G}_{L-1}] = p_L$, so the equality case of Assumption~\ref{ass:conservative} holds at level $L$ regardless.

\paragraph{Constants of the envelope.}
The envelope of Corollary~\ref{cor:upper-bound} scales with $1/\delta$. The bound is therefore loose at conventional confidence levels: at $\delta = 0.05$ and $\delta = 0.10$ it permits $20\times$ and $10\times$ overshoot respectively, far above the empirical ratios reported in \S\ref{sec:exp-cv} (where $\hat{P}_f / \bar{p}_f^{\,\text{SMC}}$ lies between $0.997$ and $1.002$ across the five \projname\ rulers in Table~\ref{tab:exp2-headtohead}) and \S\ref{sec:exp-llm} (mean $|\hat{p}_f / p_f^{\mathrm{ref}} - 1| \approx 0.026$ for PC1/$p75$). The envelope should therefore be read as a qualitative validity guarantee that overshoot is controlled uniformly across the cascade path, rather than as a tight quantitative bound. Tighter envelopes via mixture supermartingales~\cite{howard2021time} or by adapting the central limit theorem of Cérou et al.~\cite{cerou2019historical} to the data-driven setting are left to future work.

\paragraph{Calibration constant.}
Equation~\eqref{eq:cal-theory} multiplies $\hat{P}_{\ell^*}$ by $C_\tau = \mathrm{precision}(\tau)/\mathrm{recall}(\tau)$. Because $C_\tau$ is a deterministic function of the fixed offline dataset, it is $\mathcal{G}_0$-measurable, and the building-block lemma of \S\ref{app:martingale:defs} implies $\{C_\tau M_l\}$ is a non-negative supermartingale with respect to $\{\mathcal{G}_l\}$. Consequently Ville's inequality applies to the calibrated process, giving $\mathbb{P}(\hat{P}_f^{\,\mathrm{cal}} \geq C_\tau P_{\ell^*}/\delta) \leq \delta$: calibration does not break anytime validity, it shifts the envelope by the known constant~$C_\tau$.

\paragraph{Universal SS-mechanism requirement.}
The anytime bound relies on the SS-mechanism condition of \S\ref{sec:method-thresholds}, namely that the last-level event coincides with the true failure event $\mathcal{F}$ so that $\hat{p}_L = K_L/N$ is an unbiased estimate of $p_L$ under MCMC stationarity, while Assumption~\ref{ass:conservative} controls the intermediate-level factors. No precision/recall correction is required for the bound to hold, and the universal statement of anytime validity therefore applies to every \projname\ run in which the SS-mechanism is satisfied.

\paragraph{Summary.}
Theorem~\ref{thm:supermartingale} and Corollary~\ref{cor:anytime} together guarantee that \projname\ failure-probability estimates are anytime-valid: one may inspect the running estimate at any level, stop early when a pre-specified confidence or detection criterion is met, and still trust that the reported probability is a statistically valid upper bound. The data-driven design of intermediate events and thresholds does not invalidate the guarantee, provided the MCMC proposals satisfy detailed balance and the chains mix adequately. The optional precision/recall calibration preserves the guarantee up to a known multiplicative constant. Both conditions are verified empirically in \S\ref{sec:exp-cv} and \S\ref{sec:exp-llm}.

%% file: Sections/Appendix_Runtime.tex
\section{Cross-Device Runtime Calibration and Hardware}
\label{app:runtime-calibration}

The vision experiments ran across three heterogeneous machines, so raw wall-clock
numbers are not directly comparable. All runtimes quoted in
Section~\ref{sec:exp-cv} are therefore reported in a single calibrated unit,
using Traditional Subset Simulation on Dawn CPU nodes as the reference.

\paragraph{Hardware inventory.}
\begin{itemize}[leftmargin=*,itemsep=1pt,topsep=2pt]
  \item \textbf{Dawn CPU (reference).} Intel Xeon CPU partition on the Dawn
        cluster. Used for Traditional SS runs. Single core per seed, no GPU.
  \item \textbf{Workstation GPU.} Local workstation with an NVIDIA GPU used
        for \projname\ ruler-screening and encoder forward passes.
  \item \textbf{Laptop CPU.} Local laptop used for SMC sanity runs on
        subsets of seeds.
\end{itemize}
The three devices differ in clock rate, memory bandwidth, and I/O, so per-sample
wall time varies by method-device pair. Within a fixed (method, device) pair we
treat the per-sample wall time as a constant.

\paragraph{Reference rate.}
The Dawn SS reference rate is measured from a single long run:
\begin{equation*}
  T_{\mathrm{ref}} \;=\; \frac{289.46~\text{s}}{447{,}160~\text{queries}}
  \;=\; 6.473\times10^{-4}~\text{s/query}.
\end{equation*}
This rate is treated as the canonical cost of one latent-space query and is used
to project every method's query count onto a common wall-clock axis.

\paragraph{Calibration rule.}
For each method $m$ we report
\begin{equation*}
  \widehat{\mathrm{runtime}}_m \;=\; \overline{\mathrm{queries}}_m \cdot T_{\mathrm{ref}},
\end{equation*}
where $\overline{\mathrm{queries}}_m$ is the mean number of forward-model queries
per seed. The raw \texttt{runtime\_mean} column in
\texttt{exp2\_cv\_method\_summary.csv} is never used directly for \projname\ or SMC
because those runs were produced on non-reference hardware.

\paragraph{Calibrated numbers used in the main text.}
\begin{itemize}[leftmargin=*,itemsep=1pt,topsep=2pt]
  \item Traditional SS (Dawn, reference): $289.5$\,s / seed.
  \item \projname\ top-5 rulers: $428$--$463$\,s / seed (Table~\ref{tab:exp2-headtohead}).
  \item SMC ground truth ($10^6$ perturbations): $647.3$\,s / seed.
\end{itemize}

\paragraph{Assumptions and caveats.}
The calibration assumes a constant per-sample cost on the reference device,
which is accurate for the fixed-architecture MNIST classifier used in Part~1
but would need re-measurement for encoders with variable-cost forward passes.
Cross-device ratios are treated as multiplicative constants and do not account
for queue-wait time on Dawn. The reported runtimes therefore reflect pure
compute cost, not end-to-end time-to-result.

\paragraph{Reproducibility.}
The calibration script
(\texttt{Experiments/Part\,one/calibrate\_runtime.py}) reads the per-method
query counts from \texttt{exp2\_cv\_method\_summary.csv}, applies
$T_{\mathrm{ref}}$, and writes the calibrated runtimes to
\texttt{smc\_runtime\_calibrated.csv}. Re-running the script with a new
reference rate reproduces every runtime column in
Table~\ref{tab:exp2-headtohead}.

%% file: Sections/Appendix_LLM_JB.tex

\section{LLM Transfer Challenges}
\label{app:llm-transfer}

This appendix documents the structural transfer challenges and diagnostics that motivate the calibrated 1D-projection design in Section~\ref{sec:exp-llm}.

\subsection{The Mahalanobis Ruler is Structurally Broken}
\label{app:mah-broken}

The vision experiment's top-performing ruler family, Mahalanobis distance to the failure centroid (R8 in Appendix~\ref{app:rulers}), fails
categorically in LLM latent space. Cohen's $d$ across the five
turns ranges from $-0.14$ to $+0.19$, effectively zero. The apparent
$0.7$ ratio produced under the inherited configuration is a count
coincidence: the ruler concentrates records near the
$p_{\text{unsafe}}$-weighted centroid (typically
$p_{\text{unsafe}} \approx 0.42$--$0.77$), which is not the actual
jailbreak region ($p_{\text{unsafe}} > 0.80$).

The root cause is noise-dimension dominance. $56$ of $64$ projected
dimensions are dominated by the regularisation floor ($10^{-5}$).
Inverse-covariance eigenvalues reach ${\sim}10^{5}$ in these noise
dimensions against ${\sim}60$ in the data dimensions. The
Mahalanobis metric therefore amplifies noise and is direction
agnostic. We evaluated five centroid-reweighting variants (power,
exponential, jailbreak-only, contrast, and a hybrid). None fixes the
amplification.
This failure is consistent with, rather than contradicting, the general
principle that distance-based metrics work well in supervised
classifier feature spaces~\citep{lee2018mahalanobis}: Llama-Guard's
projection head is trained for next-token classification rather than
for class-conditional Gaussian feature regularisation, so the
conditions under which Mahalanobis is well-behaved are not met here.

\subsection{Recall and Precision}
\label{app:recall}

Treating $\{G(f_\theta(x)) \geq \tau\}$ as the ruler's positive region and $\mathcal{D}_{\text{bad}}$ as the ground-truth failure set,
\begin{align}
\label{eq:prec-recall-def}
\mathrm{precision}(\tau) &\;=\; \frac{\bigl|\{x \in \mathcal{D}_{\text{bad}} : G(f_\theta(x)) \geq \tau\}\bigr|}{\bigl|\{x \in \mathcal{D}_{\text{good}} \cup \mathcal{D}_{\text{bad}} : G(f_\theta(x)) \geq \tau\}\bigr|}, \\
\mathrm{recall}(\tau) &\;=\; \frac{\bigl|\{x \in \mathcal{D}_{\text{bad}} : G(f_\theta(x)) \geq \tau\}\bigr|}{|\mathcal{D}_{\text{bad}}|}, \nonumber
\end{align}
i.e.\ the fraction of above-threshold embeddings that are genuine failures, and the fraction of labelled failures retained at~$\tau$, respectively.

The correction is algebraic rather than empirical: $\mathbb{P}(G \geq \tau) \cdot \mathbb{P}(\mathcal{F} \mid G \geq \tau) = \mathbb{P}(\mathcal{F} \wedge G \geq \tau) = \mathbb{P}(\mathcal{F}) \cdot \mathrm{recall}(\tau)$, so dividing by $\mathrm{recall}(\tau)$ recovers $\mathbb{P}(\mathcal{F})$. Precision and recall are constants computed once from the labelled embeddings and do not depend on the test-time population; the identity is exact under the assumption that these two quantities are invariant across that population, which holds when the benign mass sits far from~$\tau$, the regime in which linear-projection rulers operate.

\subsection{More Samples Cannot Fix a Ruler Problem}
\label{app:n-cannot-fix}

An initial population grid $N \in \{20\text{k}, 50\text{k},
100\text{k}, 200\text{k}, 500\text{k}\}$ confirms the classical
$\text{SE} \propto 1/\!\sqrt{N}$ scaling. It also confirms that
variance reduction does not correct systematic bias. The ratio
centroid under the Mahalanobis ruler stays at about $0.7$ across all
$N$. The problem is ruler quality, not sample size.

\subsection{GCG-Style Transfer}
\label{app:gcg-transfer}

Table~\ref{tab:gcg-transfer} reports the GCG-style transfer
check referenced in Section~\ref{sec:exp-llm}.
Targets are fixed by the constructed latent fleet: the jailbreak
target is $\eta \cdot 0.076$.
Thus $\eta=10^{-3}$ has target $7.60\times10^{-5}$, and
$\eta=10^{-4}$ has target $7.60\times10^{-6}$.
The same PAIR-derived $p75$ score directions and percentile
thresholds are reused.
Two calibrations are compared: PAIR-inherited calibration keeps
the PAIR precision/recall scale, while GCG-refit calibration
recomputes precision and recall on the GCG latent distribution.
The table shows that the score/threshold ruler transfers, but
calibration must be refit.

\begin{table}[h]
\centering
\scriptsize
\resizebox{\linewidth}{!}{%
\begin{tabular}{llrrlccc}
\toprule
Run & $\eta$ & $N$ & Rounds & Ruler
& PAIR-inherited & GCG-refit (95\% CI) & Refit/target \\
\midrule
$\eta=10^{-3}$
& $10^{-3}$ & $20{,}000$ & $20$ & PC1/$p75$
& $4.38{\times}10^{-5}$ ($0.58{\times}$)
& $7.82{\times}10^{-5}$
  [$6.38{\times}10^{-5}$, $9.27{\times}10^{-5}$]
& $1.03{\times}$ \\

$\eta=10^{-3}$
& $10^{-3}$ & $20{,}000$ & $20$ & Centroid/$p75$
& $4.74{\times}10^{-5}$ ($0.62{\times}$)
& $7.92{\times}10^{-5}$
  [$6.68{\times}10^{-5}$, $9.16{\times}10^{-5}$]
& $1.04{\times}$ \\

$\eta=10^{-4}$
& $10^{-4}$ & $20{,}000$ & $20$ & PC1/$p75$
& $6.18{\times}10^{-6}$ ($0.81{\times}$)
& $1.11{\times}10^{-5}$
  [$5.87{\times}10^{-6}$, $1.62{\times}10^{-5}$]
& $1.45{\times}$ \\

$\eta=10^{-4}$
& $10^{-4}$ & $20{,}000$ & $20$ & Centroid/$p75$
& $5.20{\times}10^{-6}$ ($0.68{\times}$)
& $8.70{\times}10^{-6}$
  [$4.21{\times}10^{-6}$, $1.32{\times}10^{-5}$]
& $1.14{\times}$ \\

$\eta=10^{-4}$, high budget
& $10^{-4}$ & $200{,}000$ & $60$ & PC1/$p75$
& $4.53{\times}10^{-6}$ ($0.60{\times}$)
& $8.10{\times}10^{-6}$
  [$7.22{\times}10^{-6}$, $8.99{\times}10^{-6}$]
& $1.07{\times}$ \\

$\eta=10^{-4}$, high budget
& $10^{-4}$ & $200{,}000$ & $60$ & Centroid/$p75$
& $4.65{\times}10^{-6}$ ($0.61{\times}$)
& $7.77{\times}10^{-6}$
  [$6.81{\times}10^{-6}$, $8.73{\times}10^{-6}$]
& $1.02{\times}$ \\
\bottomrule
\end{tabular}%
}
\caption{GCG-style calibrated jailbreak estimates. PAIR-inherited
calibration reuses the PAIR precision/recall scale; GCG-refit
calibration recomputes precision and recall on the GCG latent
distribution. The 95\% intervals propagate SS threshold-event
uncertainty after refit calibration, but do not include additional
precision/recall-refit uncertainty.}
\label{tab:gcg-transfer}
\end{table}

At $\eta=10^{-3}$, GCG-refit calibration recovers the target well:
PC1/$p75$ is $1.03{\times}$ target and Centroid/$p75$ is
$1.04{\times}$ target.
The inherited PAIR calibration is the weak link, giving only
$0.58{\times}$--$0.62{\times}$ target.
At $\eta=10^{-4}$, the $N=20{,}000$, $20$-round run is useful as a
noise diagnostic but is visibly less stable.
The high-budget $N=200{,}000$, $60$-round run gives the cleaner
rare-event check: after refitting calibration, PC1/$p75$ is
$1.07{\times}$ target and Centroid/$p75$ is $1.02{\times}$ target,
whereas PAIR-inherited calibration remains near
$0.60{\times}$--$0.61{\times}$ target.
The practical lesson is that score directions and percentile
thresholds may transfer across attack styles, but the calibration
constant $C_\tau$ is target-distribution specific.


%% file: Sections/Appendix_Ruler_Selection.tex
\section{Ruler Ranking Details}
\label{app:llm-ranking}

This appendix validates the directional-KL ruler selector used in
Section~\ref{sec:exp-llm}.
The selection problem is family-level: choose the ruler direction
before running rare-event simulations.
Threshold percentiles are then handled by the PC1 sweep in
Section~\ref{sec:exp-llm}.

\subsection{Six-Family Ground Truth}

Six ruler families are evaluated: five learned 1D projection
families plus the inherited Mahalanobis baseline.
Each learned family is swept over $p50$, $p75$, $p90$, and $p95$.
The ground-truth error is the mean relative error
$\overline{|r-1|}$ averaged over five PAIR turns and
$\eta \in \{10^{-3},10^{-2},5{\times}10^{-2},10^{-1}\}$.
The $\eta=10^{-4}$ cells are excluded from the family-ranking
target because they are dominated by finite-sample discreteness.

\begin{table}[h]
\centering
\small
\caption{Ground-truth family ranking and selector values. Each
family is represented by its best threshold variant. Higher
worst-turn directional KL is better; lower error is better.}
\label{tab:selector-family-truth}
\begin{tabular}{llrrrr}
\toprule
Family & Best variant & $\overline{|r-1|}$ & Worst $|r-1|$
& Worst KL & Cohen's $d$ \\
\midrule
\textbf{PC1}      & \texttt{pc1\_p75}      & $\mathbf{0.026}$ & $0.120$ & $\mathbf{2.913}$ & $1.510$ \\
Centroid          & \texttt{centroid\_p75} & $0.028$ & $0.120$ & $2.822$ & $1.526$ \\
Fisher LDA        & \texttt{fisher\_p95}   & $0.172$ & $1.238$ & $0.960$ & $1.619$ \\
Mahalanobis       & \texttt{mah\_bad}      & $0.566$ & $0.958$ & $0.043$ & $0.000$ \\
Logistic Reg.     & \texttt{logreg\_p95}   & $1.151$ & $19.941$ & $1.126$ & $1.726$ \\
CentroidRot90     & \texttt{centroid\_rot90\_p75}
                  & $39.962$ & $225.10$ & $0.037$ & $0.198$ \\
\bottomrule
\end{tabular}
\end{table}

The two simple linear projections, PC1 and Centroid, are the only
families with mean relative error below $3\%$.
Cohen's $d$ ranks Logistic Regression highest, even though it is
more than an order of magnitude worse downstream than PC1.
This motivates using an asymmetric selector tied to the actual
error mechanism.

\begin{figure}[h]
\centering
\includegraphics[width=0.72\linewidth]
{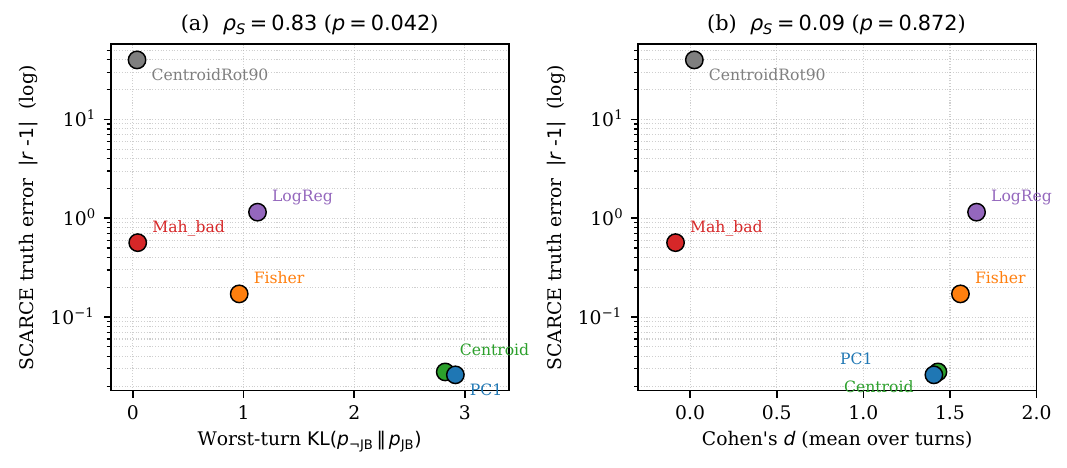}
\caption{Selector diagnostics against downstream ruler-family error.
Worst-turn directional KL is monotone in the true family error and
selects PC1, while Cohen's $d$ rewards in-distribution separation
and selects Logistic Regression.}
\label{fig:selector-kl-vs-cohens}
\end{figure}

\subsection{Why the Direction Matters}

At small $\eta$, the fleet is almost entirely benign.
The dominant error mode is therefore one-sided leakage: benign
mass crossing the surrogate threshold into the jailbreak region.
Let $P_b(\tau)$ be benign mass above threshold, $P_h(\tau)$ be
harmful mass above threshold, and
$\pi=P(\mathrm{JB}\mid\mathrm{PAIR})$.
The calibrated estimate has relative bias
\begin{equation}
\frac{\hat p_{\mathrm{cor}}-\eta\pi}{\eta\pi}
=
\frac{1-\eta}{\eta}
\cdot
\frac{P_b(\tau)}{P_h(\tau)} .
\label{eq:selector-bias-decomp}
\end{equation}
At $\eta=10^{-3}$, the prefactor is approximately $1000$, so a
benign-leakage ratio of only $10^{-3}$ can create $100\%$
relative error.
A symmetric metric averages both directions and can miss this
failure mode.
The directional KL
$\mathrm{KL}(p_{\mathrm{good}}\|p_{\mathrm{bad}})$ instead weights
the score region where benign leakage is carried, which is why it
matches the downstream error ordering.

\subsection{Metric Ranking and Stability}

We benchmarked $14$ candidate metrics; Table~\ref{tab:selector-cv}
shows the representative family-level results used for the main
claim.
Worst-turn aggregation is used because one weak attack turn can
dominate the cascade error, whereas mean aggregation dilutes that
bottleneck.

\begin{table}[h]
\centering
\scriptsize
\caption{Selector robustness on the six-family grid. The $p$ column
is the asymptotic Spearman value; exact permutation for the
headline row is discussed below. LOFO and LOTO denote
leave-one-family-out and leave-one-turn-out minima.}
\label{tab:selector-cv}
\begin{tabular}{p{0.34\linewidth}lrrrr}
\toprule
Metric & Aggr. & $\rho$ & $p$ & LOFO min & LOTO min \\
\midrule
$\mathrm{KL}(p_{\mathrm{good}}\|p_{\mathrm{bad}})$
  & worst & $\mathbf{0.83}$ & $0.042$ & $\mathbf{0.70}$ & $\mathbf{0.71}$ \\
Bhattacharyya distance
  & worst & $0.71$ & $0.111$ & $0.50$ & $0.60$ \\
$\mathrm{KL}(p_{\mathrm{good}}\|p_{\mathrm{bad}})$
  & mean & $0.66$ & $0.156$ & $0.50$ & $0.54$ \\
Cohen's $d$
  & mean & $0.09$ & $0.872$ & $-0.40$ & $-0.14$ \\
Jensen--Shannon divergence
  & mean & $0.09$ & $0.872$ & $-0.40$ & $-0.14$ \\
AUC
  & mean & $0.09$ & $0.872$ & $-0.40$ & $-0.14$ \\
Mutual information
  & mean & $0.09$ & $0.872$ & $-0.40$ & $-0.14$ \\
Wasserstein-1
  & mean & $0.03$ & $0.957$ & $-0.70$ & $0.03$ \\
\bottomrule
\end{tabular}
\end{table}

Worst-turn directional KL ranks PC1 first and correlates with
negative downstream error at Spearman $\rho=0.83$ on the six
families.
Exact permutation over the six family labels gives $p=0.058$
two-sided and $p=0.029$ for the pre-specified positive direction.
A nonparametric bootstrap over families gives a broad $95\%$
percentile interval $[0.00,1.00]$, as expected at $n=6$.
We therefore report the bootstrap as a small-sample caution and
rely on the leave-one-family-out and leave-one-turn-out checks as
the more interpretable stability diagnostics.

The variant-level files also contain an exploratory hybrid selector
that combines family-level KL with threshold-level calibration
statistics. It ranks PC1/$p75$ first, but we do not use it as the
main claim because it mixes two decisions: selecting the ruler
family and selecting the operating percentile.

%% file: Sections/Appendix_Plots.tex
\section{Additional Plots}
\label{app:plots}

This appendix collects supplementary figures that support the efficiency and
reliability claims in Section~\ref{sec:exp-cv} and the scaling sweeps in
Section~\ref{sec:exp-llm}. They are omitted from the main body to keep the
discussion focused on the headline comparisons.

\subsection*{Part~I: Vision Efficiency}
\label{app:plots-part1}

\begin{figure}[H]
    \centering
    \subfloat[$\log_{10}\mae(p_f)$ over the 13 candidate rulers. \label{fig:app-exp1-mae}]{
    \includegraphics[width=0.48\linewidth]{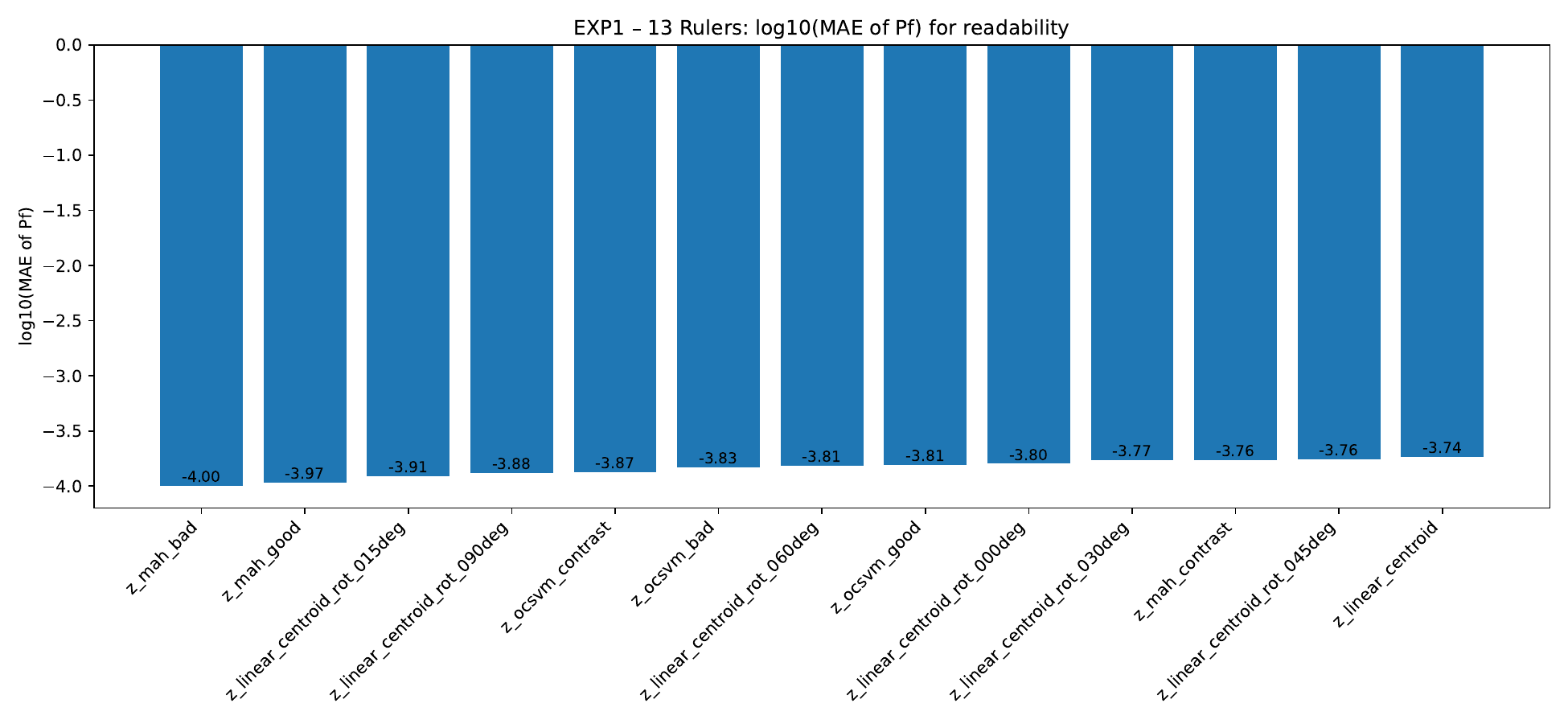}
    }
    \hfill
    \subfloat[Missing rare-event hits over 50 flipped seeds.\label{fig:app-exp1-missing}]{
    \includegraphics[width=0.48\linewidth]{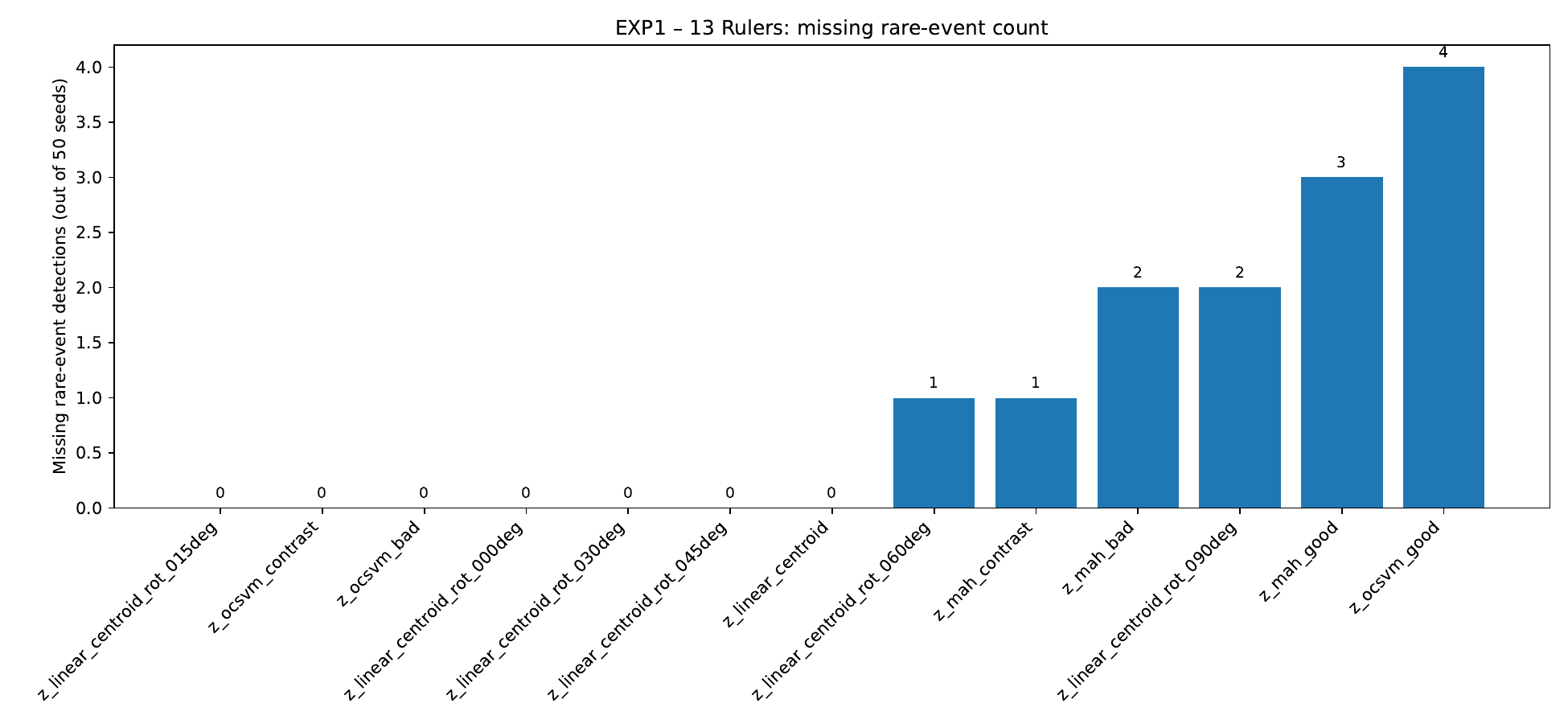}
    }
    \caption{Step~1 ruler pre-screening on 50 flipped MNIST seeds. Bad-anchored variants dominate the 13-ruler catalogue in both estimation error and rare-event recall, motivating the top-5 ruler set used in Section~\ref{sec:exp-cv}.}

    \label{fig:app-exp1-screening}
\end{figure}

\begin{figure}[H]
    \centering
    \subfloat[Mean query count per ruler. \label{fig:app-exp1-queries}]{
    \includegraphics[width=0.48\linewidth]{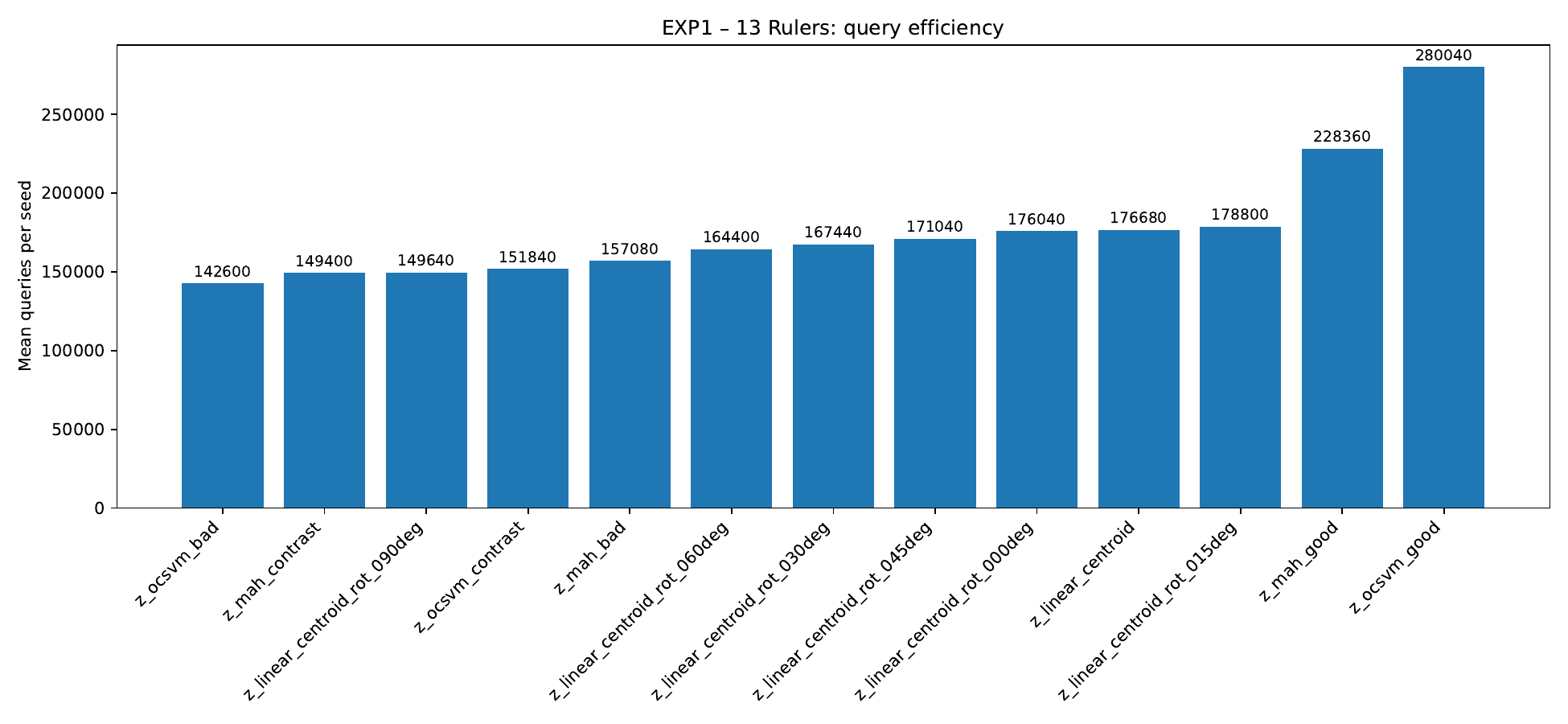}
    }
    \hfill
    \subfloat[Calibrated mean runtime per ruler.\label{fig:app-exp1-runtime}]{
    \includegraphics[width=0.48\linewidth]{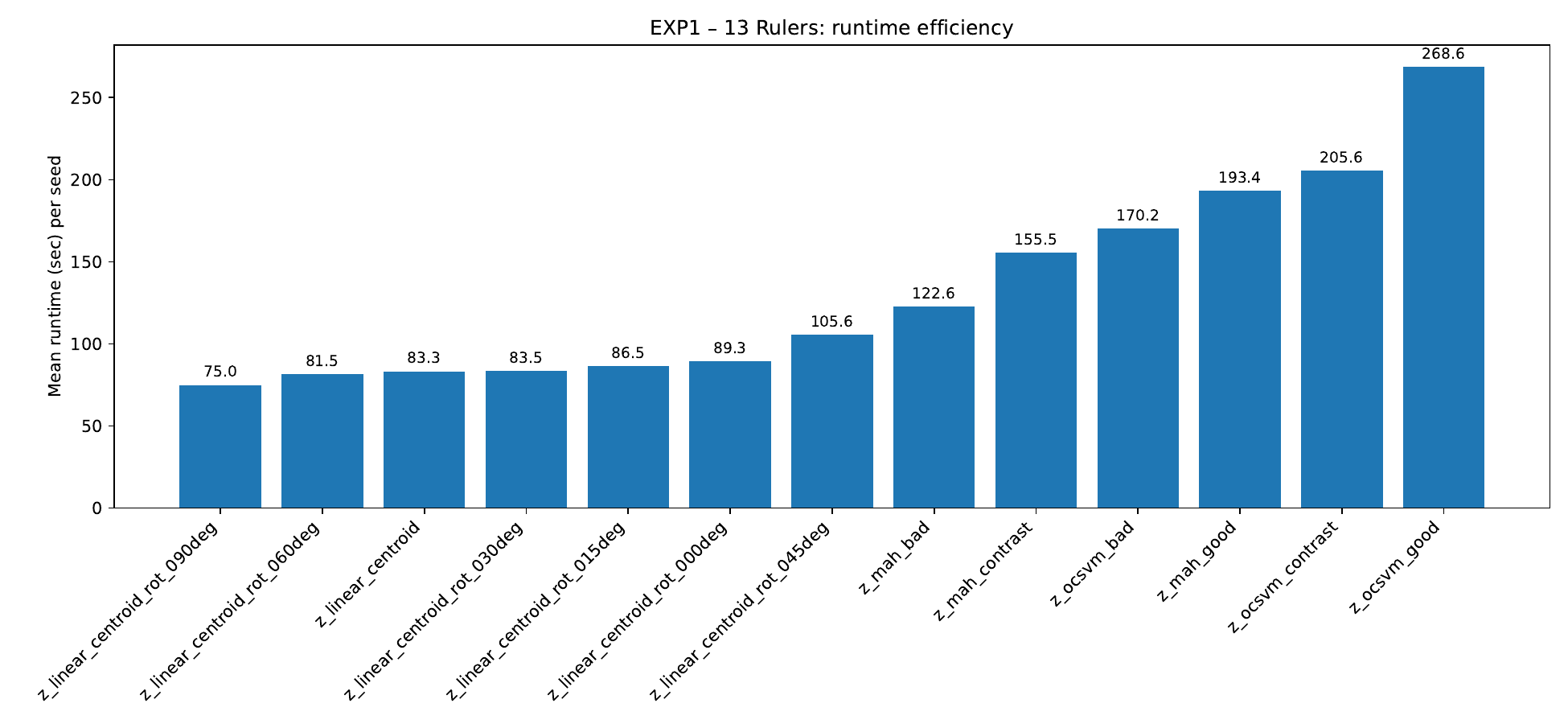}
    }
    \caption{Step~1 efficiency budgets. The rulers selected by accuracy and reliability also avoid the largest query and runtime costs, so the screening stage identifies a practical Pareto region rather than a purely accurate but expensive corner.}

    \label{fig:app-exp1-efficiency-bars}
\end{figure}

\begin{figure}[H]
    \centering
    \subfloat[Queries versus $\mae(p_f)$. \label{fig:app-exp1-query-scatter}]{
    \includegraphics[width=0.48\linewidth]{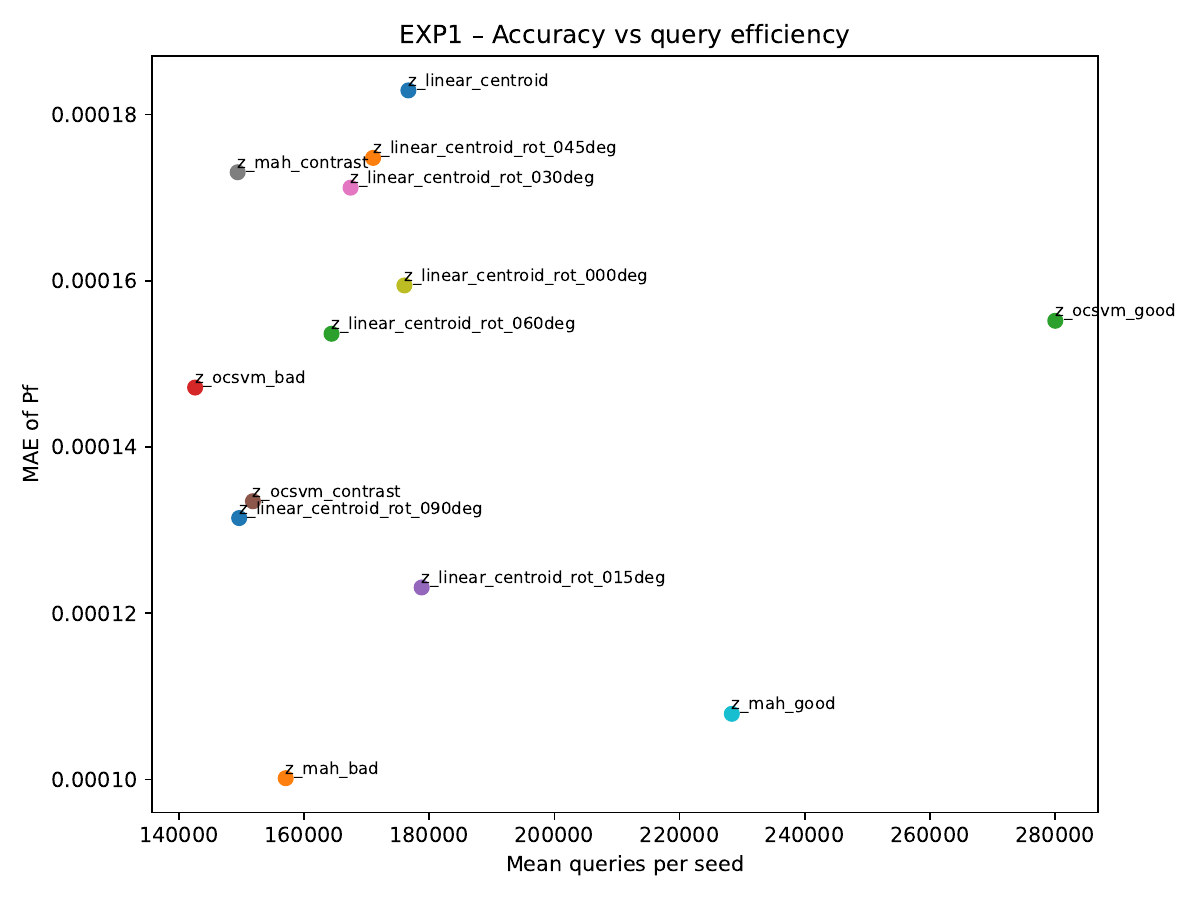}
    }
    \hfill
    \subfloat[Runtime versus $\mae(p_f)$.\label{fig:app-exp1-runtime-scatter}]{
    \includegraphics[width=0.48\linewidth]{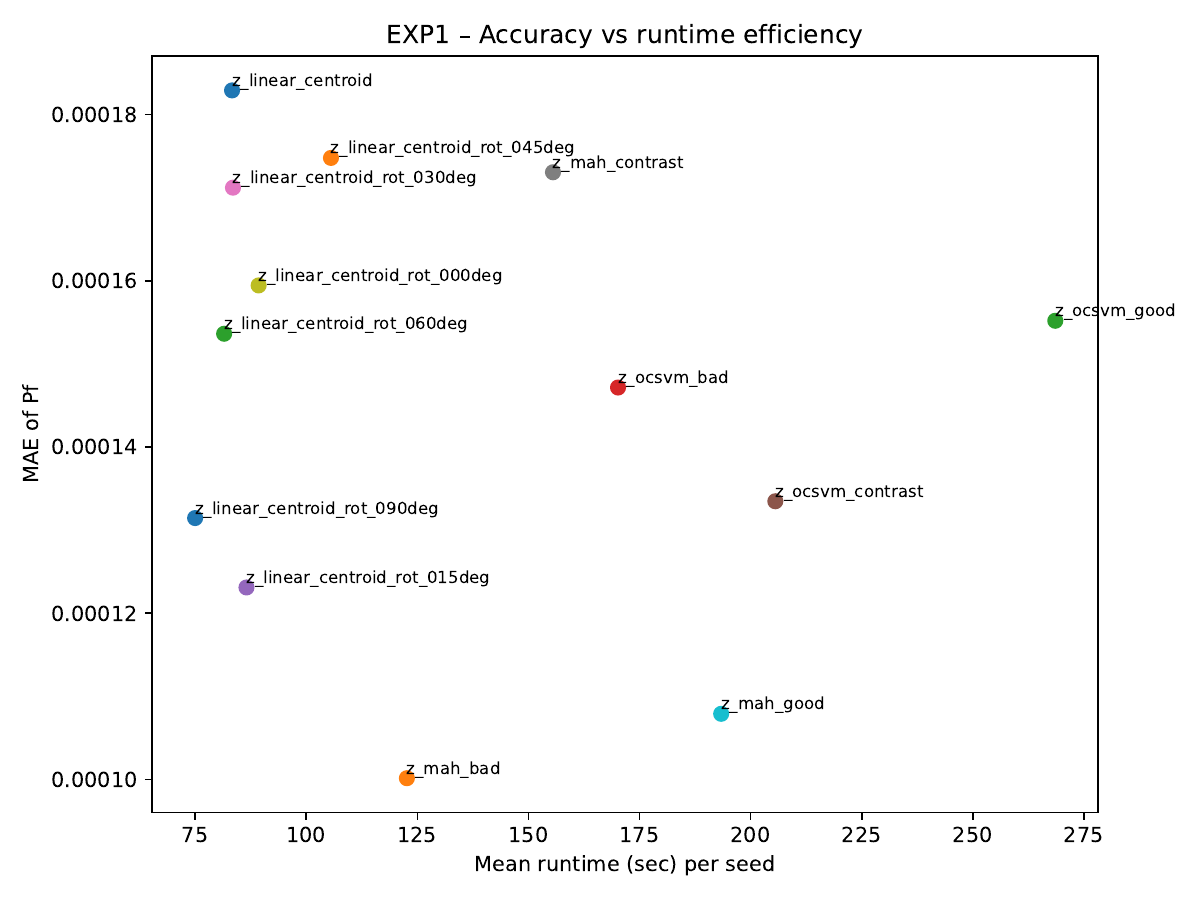}
    }
    \caption{Step~1 Pareto views. Bad-anchored geometric rulers occupy the low-error and
    moderate-budget region, while weakly aligned rulers are both less accurate and more expensive.}

    \label{fig:app-exp1-pareto}
\end{figure}

\begin{figure}[H]
    \centering
    \includegraphics[width=0.75\linewidth]{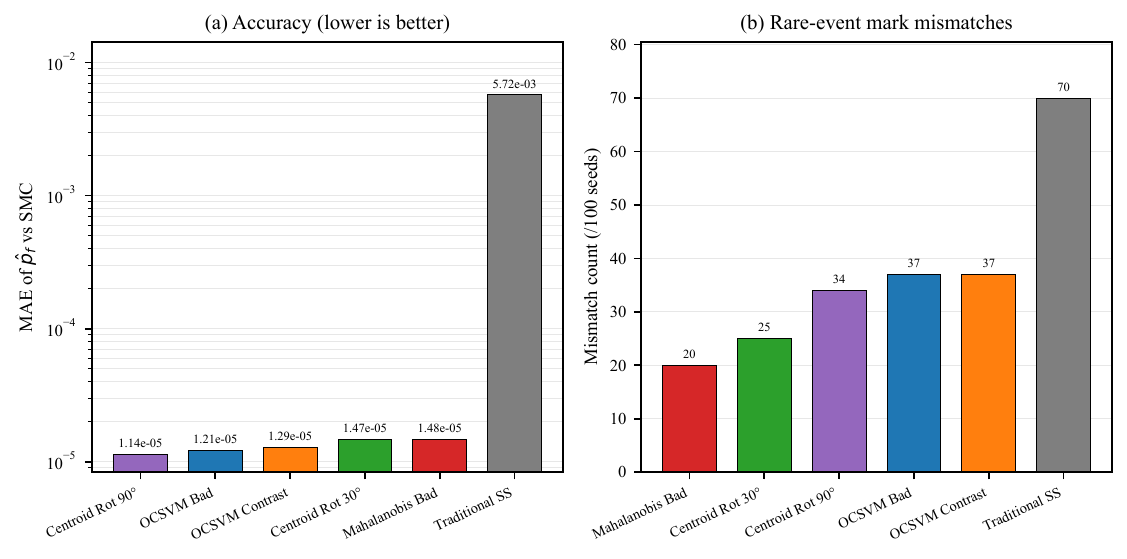}
    \caption{Step~2 MNIST comparison with accuracy and mismatch counts. \projname\ rulers keep absolute error near the SMC reference and produce far fewer miscounts.}
    \label{fig:app-exp2-accuracy-mismatch}
\end{figure}

\begin{figure}[H]
    \centering
    \includegraphics[width=0.75\linewidth]{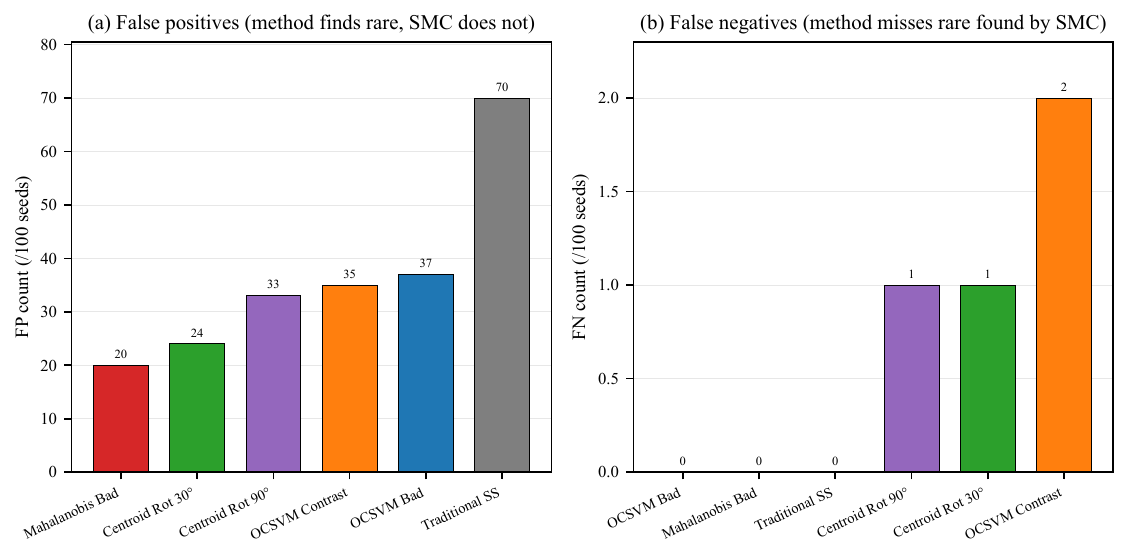}
    \caption{Miscounts over 100 held-out seeds. Traditional SS attains zero false negatives at the cost of 70 false positives, showing that its apparent efficiency is caused by premature threshold acceptance rather than a better rare-event estimate.}
    \label{fig:app-exp2-fp-fn}
\end{figure}

\begin{figure}[H]
  \centering
  \includegraphics[width=0.7\linewidth]{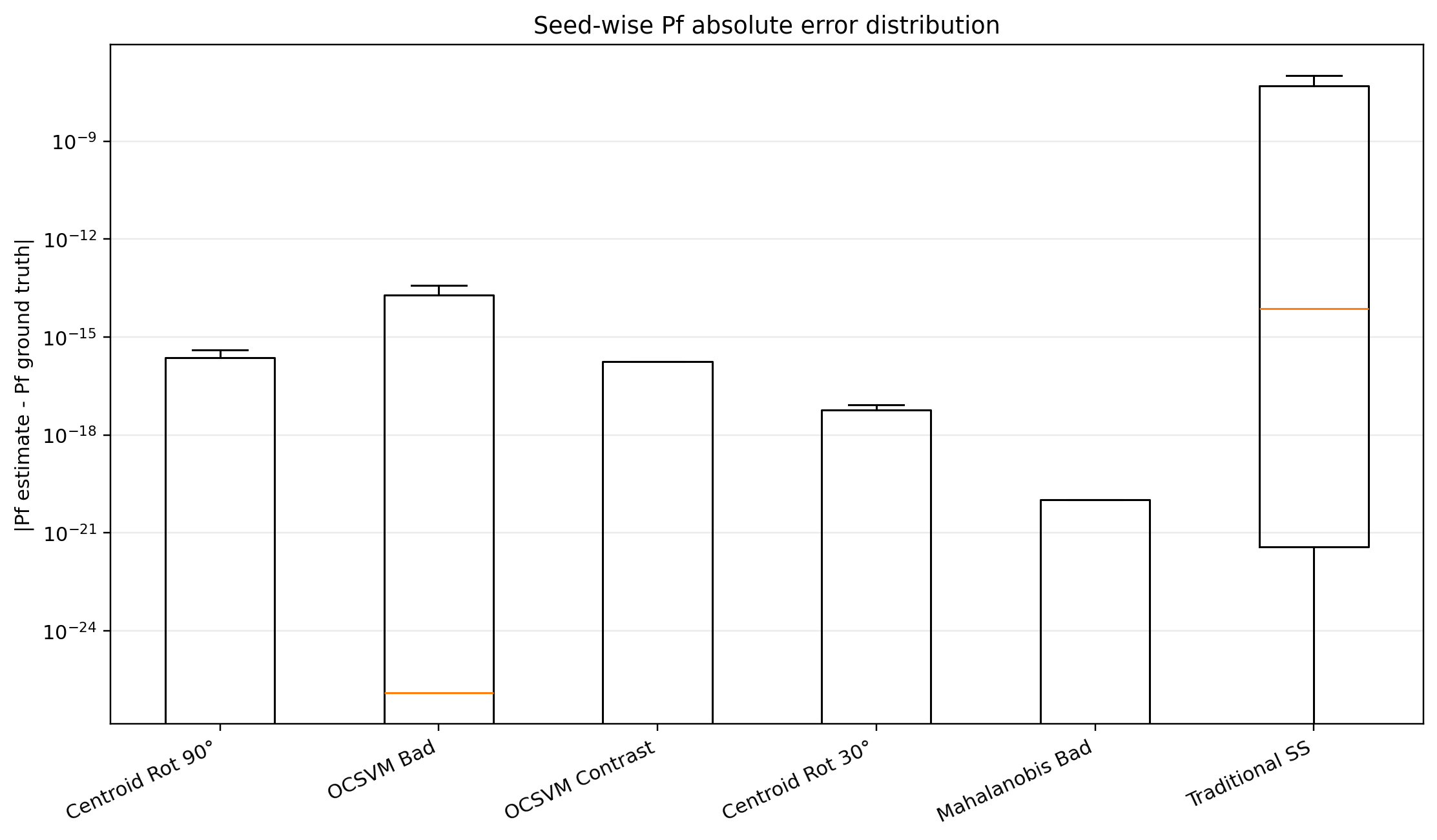}
  \caption{Seed-wise distribution of $|\hat p_f - p_f^{\mathrm{SMC}}|$ across
    $100$ MNIST seeds for the top \projname\ rulers and Traditional SS. The \projname\
    distributions are tightly concentrated near zero while Traditional SS
    shows a heavy upper tail driven by systematic false positives.}
  \label{fig:app-exp2-boxplot}
\end{figure}

\subsection*{Part~II: LLM Jailbreak Sweeps}
\label{app:plots-part2}

\begin{figure}[H]
  \centering
  \includegraphics[width=0.78\linewidth]{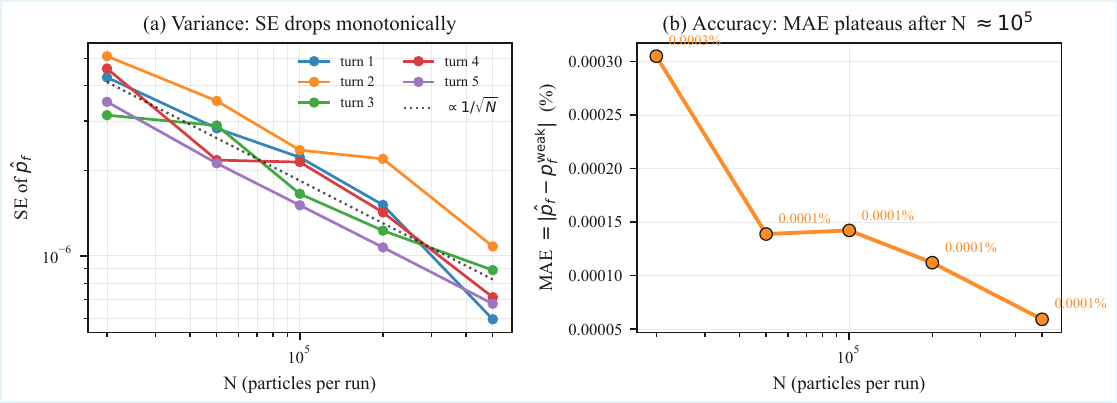}
  \caption{Population-size sweep for PC1/$p75$. Increasing $N$ reduces Monte
  Carlo variation as expected, while the mean error plateaus once the sample
  budget is large enough for stable cascade calibration.}
  \label{fig:app-llm-n-grid}
\end{figure}